\documentclass[letterpaper]{article}
\usepackage{uai2020}
\usepackage[margin=1in]{geometry}

\usepackage{times}

\usepackage[round]{natbib} 

\usepackage[utf8]{inputenc}
\usepackage[english]{babel}
\usepackage[T1]{fontenc}

\usepackage{microtype}
\usepackage[colorlinks,linkcolor=red,citecolor=blue]{hyperref}

\usepackage{amsthm}
\usepackage{mathrsfs}
\usepackage{amsmath}
\usepackage{amssymb}
\usepackage{dsfont}

\usepackage{algorithm}
\usepackage{algpseudocode}
\algrenewcommand\algorithmicrequire{\textbf{Input:}}
\algrenewcommand\algorithmicensure{\textbf{Output:}}

\usepackage{url}
\usepackage{graphicx}

\usepackage{subcaption} 

\theoremstyle{plain}
\newtheorem{theorem}{Theorem}[section]
\newtheorem{lemma}[theorem]{Lemma}

\newtheorem{proposition}[theorem]{Proposition}

\theoremstyle{remark}

\newcommand{\mean}{\mathbb{E}}
\newcommand{\med}{\text{med}}

\newcommand{\ud}{\mathop{}\!\mathrm{d}}
\newcommand{\cond}{\,|\,}

\DeclareMathOperator{\Normal}{\mathcal{N}}

\DeclareMathOperator{\GP}{\mathcal{GP}}
\DeclareMathOperator{\Unif}{\mathcal{U}}

\newcommand{\Bb}{\mathbf{b}}

\newcommand{\Bh}{\mathbf{h}}
\newcommand{\Bk}{\mathbf{k}}

\newcommand{\Bs}{\mathbf{s}}

\newcommand{\Bx}{\mathbf{x}}

\newcommand{\BB}{\mathbf{B}}

\newcommand{\BH}{\mathbf{H}}
\newcommand{\BK}{\mathbf{K}}

\newcommand{\BR}{\mathbf{R}}

\newcommand{\BW}{\mathbf{W}}
\newcommand{\BSigma}{\boldsymbol{\Sigma}}

\newcommand{\Bpsi}{\boldsymbol{\psi}}

\newcommand{\Btheta}{\boldsymbol{\theta}}
\newcommand{\Bvartheta}{\boldsymbol{\vartheta}}

\newcommand{\Bphi}{\boldsymbol{\phi}}

\renewcommand{\epsilon}{\varepsilon}
\newcommand{\Bgamma}{\boldsymbol{\gamma}}
\newcommand{\BDelta}{\boldsymbol{\Delta}}

\newcommand{\reals}{\mathbb{R}}

\newcommand{\Dset}{\mathscr{D}}

\newcommand*{\half}[1][2]{\frac{1}{#1}}

\newcommand{\Bzeros}{\mathbf{0}}
\newcommand{\Id}{\mathbf{I}}

\newcommand{\bvn}{\textnormal{BvN}} 

\newcommand{\myquad}{\quad\quad}

\newcommand{\Deltav}{\tau^2}
\newcommand{\lik}{f}
\newcommand{\eqdef}{\triangleq}
\newcommand{\T}{^{\top}}

\newcommand{\maxv}{\textnormal{MAXV}}
\newcommand{\maxmad}{\textnormal{MAXMAD}}
\newcommand{\eiv}{\textnormal{EIV}}
\newcommand{\eimad}{\textnormal{EIMAD}}

\newcommand{\bigO}{\mathcal{O}} 
\newcommand{\indic}{\mathds{1}}
\newcommand{\simiid}{\overset{\text{i.i.d.}}{\sim}}
\newcommand{\opt}{^{\text{opt}}}

\newcommand{\probmeas}{\Pi}

\newcommand{\piabc}{\pi_{\textnormal{ABC}}}

\newcommand{\tildepiabc}{\tilde{\pi}_{\textnormal{ABC}}}

\newcommand{\pie}{\pi_{\epsilon}}

\newcommand{\Bxobs}{\Bx_{\textnormal{o}}}
\newcommand{\Bsobs}{\Bs_{\textnormal{o}}}
\newcommand{\babc}{Bayesian ABC}

\newcommand{\owen}{Owen's T function}
\newcommand{\iter}{i} 

\renewcommand\cite{\citep}

\newcommand{\appe}{\textnormal{Appendix}}

\newcommand{\argmax}[1]{\underset{#1}{\operatorname{arg}\operatorname{max}}\;}
\newcommand{\argmin}[1]{\underset{#1}{\operatorname{arg}\operatorname{min}}\;}

\def\app#1#2{%
  \mathrel{%
    \setbox0=\hbox{$#1\sim$}%
    \setbox2=\hbox{%
      \rlap{\hbox{$#1\propto$}}%
      \lower1.1\ht0\box0%
    }%
    \raise0.25\ht2\box2%
  }%
}

\allowdisplaybreaks 

\title{Batch simulations and uncertainty quantification in Gaussian process surrogate approximate Bayesian computation}

\author{{\bf Marko J\"{a}rvenp\"{a}\"{a}}\thanks{Current address: Department of Biostatistics, University of Oslo, email: \texttt{m.j.jarvenpaa@medisin.uio.no}}, Aki Vehtari, Pekka Marttinen \\
Helsinki Institute for Information Technology HIIT, Department of Computer Science, 
Aalto University \\
}

\begin{document}

\maketitle

\begin{abstract}
The computational efficiency of approximate Bayesian computation (ABC) has been improved by using surrogate models such as Gaussian processes (GP). In one such promising framework the discrepancy between the simulated and observed data is modelled with a GP which is further used to form a model-based estimator for the intractable posterior. 
In this article we improve this approach in several ways.
We develop batch-sequential Bayesian experimental design strategies to parallellise the expensive simulations. In earlier work only sequential strategies have been used. 
Current surrogate-based ABC methods also do not fully account the uncertainty due to the limited budget of simulations as they output only a point estimate of the ABC posterior. We propose a numerical method to fully quantify the uncertainty in, for example, ABC posterior moments. We also provide some new analysis on the GP modelling assumptions in the resulting improved framework called \babc{} and discuss its connection to Bayesian quadrature (BQ) and Bayesian optimisation (BO). Experiments with toy and real-world simulation models demonstrate advantages of the proposed techniques. 
\end{abstract}

\section{INTRODUCTION}

Approximate Bayesian computation \cite{Beaumont2002,Marin2012} is used for Bayesian inference when the likelihood function of a statistical model of interest is intractable, i.e.,~when the analytical form of the likelihood is either unavailable or too costly to evaluate, but simulating the model is feasible. 
The main idea of the ABC rejection sampler \cite{Pritchard1999} is to draw a parameter from the prior, use it to simulate one pseudo-data set and finally accept the parameter as a draw from an approximate posterior if the discrepancy between the simulated and observed data sets is small enough. 
%
While the computational efficiency of this basic ABC algorithm has been improved in several ways, many models e.g.~in genomics and epidemiology \cite{Numminen2013,Marttinen2015}, astronomy \cite{Rogers2019} and climate science \cite{Holden2018} are expensive-to-simulate which makes the sampling-based ABC inference algorithms infeasible. 
%
To increase sample-efficiency of ABC,
various methods using surrogate models such as neural networks \cite{Papamakarios2016,Papamakarios2018,Greenberg2019} and Gaussian processes
\cite{Meeds2014,Wilkinson2014,Gutmann2016,Jarvenpaa2018_aoas,Jarvenpaa2018_acq} have been proposed. 

In one promising surrogate-based ABC framework the discrepancy between the observed and simulated data, a key quantity in ABC, is modelled with a GP \cite{Gutmann2016,Jarvenpaa2018_aoas,Jarvenpaa2018_acq}.
The GP model is then used to form an estimator for the (approximate) posterior and to adaptively select new evaluation locations. 
Sequential \emph{Bayesian experimental design} (also known as \emph{active learning}) methods to select the simulation locations so as to maximise the sample-efficiency were proposed by \citet{Jarvenpaa2018_acq}. However, their methods allow to run only one simulation at a time although in practice one often has access to multiple cores to run some of the simulations in parallel. 
In this work, we resolve this limitation by developing batch simulation methods which are then shown to considerably decrease the wall-time needed for ABC inference. 
Our approach (Section \ref{sec:parallel}) is based on a Bayesian decision theoretic framework recently developed by \citet{Jarvenpaa2019_sl} who, however, assumed that expensive and potentially noisy likelihood evaluations are available (e.g.~by synthetic likelihood method \citep{Wood2010,Price2018}). In this work we instead focus on ABC scenario where only less than a thousand model simulations can be obtained. 

In practice the posterior distribution is often summarised for further decision making using e.g.~the expectation and variance. When the computational resources for ABC inference are limited, it would be important to assess the accuracy of such summaries, but this has not been done in earlier work.
As the second main contribution of this article, we devise an approximate numerical method to propagate the uncertainty of the discrepancy, represented by the GP model, to the resulting ABC posterior summaries (Section \ref{sec:post_uncertainty}). Such uncertainty estimates are useful for assessing the accuracy of the inference and guiding the termination of the inference algorithm. 
We call the resulting improved framework as \emph{\babc{}} in analogy with the related problems of \emph{Bayesian} quadrature and \emph{Bayesian} optimisation. 

We also provide new insights on the underlying GP modelling assumptions (\appe{}~\ref{appsubsec:gaussianity}) and on the connections between \babc{}, BQ and BO to improve understanding of these conceptually similar techniques (Section \ref{sec:theory}). 
Finally, we demonstrate the ABC posterior uncertainty quantification and show that \babc{} framework is well-suited for parallel simulations using several numerical experiments (Section \ref{sec:experiments}).

\section{BRIEF BACKGROUND ON ABC} \label{subsec:abc}

%
We denote the (continuous) parameters of the statistical model of interest with $\Btheta\in\Theta\subset\reals^p$. 
The posterior distribution, that describes our knowledge of $\Btheta$ given some observed data $\Bxobs\in\mathcal{X}$ and a prior density $\pi(\Btheta)$, can then be computed using Bayes' theorem
%
\begin{align}
\pi(\Btheta \cond \Bxobs) 
&= \frac{\pi(\Btheta)\pi(\Bxobs\cond\Btheta)}{\int_{\Theta} \pi(\Btheta')\pi(\Bxobs\cond\Btheta') \ud \Btheta'}.
%
\label{eq:bayes}
\end{align}
%
If the likelihood function $\pi(\Bxobs\cond\Btheta)$ is intractable, evaluating (\ref{eq:bayes}) even up-to-normalisation becomes infeasible. Standard ABC algorithms such as the ABC rejection sampler instead target the approximate posterior 
\begin{align}
\piabc(\Btheta \!\cond\! \Bxobs)\!\eqdef\!\frac{\pi(\Btheta)\!\int_{\mathcal{X}}\!\pie(\Bxobs\!\cond\!\Bx) \pi(\Bx\!\cond\!\Btheta) \ud \Bx}
{\int_{\Theta}\!\pi(\Btheta')\!\int_{\mathcal{X}}\!\pie(\Bxobs\!\cond\!\Bx') \pi(\Bx'\!\cond\!\Btheta') \ud \Bx' \ud \Btheta'}, \label{eq:abc_post}
\end{align}
where $\pie(\Bxobs\cond\Bx) = \indic_{\Delta(\Bxobs,\Bx) \leq \epsilon}$. Other choices of kernel $\pie$ are also possible \cite{Wilkinson2013}. Above, $\Delta:\mathcal{X}^2\rightarrow\reals_+$ is the discrepancy function used to compare the similarity of the data sets and $\epsilon$ is a threshold parameter. Small $\epsilon$ produces good approximations but renders sampling-based ABC methods inefficient. 
A well-constructed discrepancy function is an important ingredient of accurate ABC inference \cite{Marin2012}. In this article we assume a suitable discrepancy function is already available (e.g.~constructed based on expert opinion, earlier analyses on other similar models, pilot runs or distance measures between raw data sets \cite{Park2016,Bernton2019}) and focus on approximating any given ABC posterior in (\ref{eq:abc_post}) as well as possible given only a limited budget of simulations.

\section{BAYESIAN ABC FRAMEWORK} \label{sec:post_babc}\label{subsec:babc}

We describe our \babc{} framework here. The main difference to earlier work \citep{Jarvenpaa2018_acq} is that we use a hierarchical GP model and, most importantly, explicitly quantify the uncertainty of the ABC posterior instead of resorting to point estimation. 
The main idea is to explicitly use another layer of Bayesian inference to estimate the ABC posterior in (\ref{eq:abc_post}). The previously simulated discrepancy-parameter-pairs are treated as data to learn a surrogate model, which will predict the discrepancy for a given parameter value. The surrogate model is further used to form an estimator for the ABC posterior in (\ref{eq:abc_post}) and to adaptively acquire new data. 

We assume that each discrepancy evaluation, denoted by $\Delta_i$ at the corresponding parameter $\Btheta_i$, is generated as
\begin{equation}
    \Delta_i = f(\Btheta_i) + \nu_i, \quad \nu_i \simiid \Normal(0,\sigma_n^2),
    \label{eq:noisemodel}
\end{equation}
where $\sigma_n^2>0$ is the variance of the discrepancy\footnote{While this modelling assumption may seem strong, it has been used successfully before. We now give a justification for this modelling choice in \appe{} \ref{appsubsec:gaussianity}.}. 
%
To encode the assumptions of smoothness and e.g.~potential quadratic shape of the discrepancy $\Delta_{\Btheta}$, in this work its unknown mean function $f$ is given a hierarchical GP prior
%
\begin{align}
\begin{split}
    f \cond \Bgamma &\sim \GP(m_0(\Btheta),k_{\Bphi}(\Btheta,\Btheta')), \\
    m_0(\Btheta) &\eqdef \sum_{i=1}^r \gamma_i h_i(\Btheta), \quad 
    \Bgamma \sim \Normal(\Bb,\BB), \label{eq:gp_prior}
\end{split}
\end{align}
where $k_{\Bphi}:\Theta^2\rightarrow\reals$ is a covariance function with hyperparameters $\Bphi$ and $h_i:\Theta\rightarrow\reals$ are basis functions (both assumed continuous). 
%
We marginalise $\Bgamma$ in (\ref{eq:gp_prior}), as in \citet{OHagan1978}, and \citet{Riihimaki2014}, 
to obtain the GP prior
\begin{equation}
    \lik \sim \GP(\Bh(\Btheta)\T\Bb,k_{\Bphi}(\Btheta,\Btheta') + \Bh(\Btheta)\T\BB \Bh(\Btheta')),
\end{equation}
where $\Bh(\Btheta)\in\reals^r$ is a column vector consisting of the basis functions $h_i$ evaluated at $\Btheta$.
For now, we assume the GP hyperparameters $\Bpsi \eqdef (\sigma_n^2,\Bphi)$ are fixed and omit $\Bpsi$ from our notation for brevity.

Given training data $D_{t}\eqdef\{(\Delta_i,\Btheta_i)\}_{i=1}^t$, we obtain $f \cond D_{t} \sim \GP(m_{t}(\Btheta),c_{t}(\Btheta, \Btheta'))$, 
\begin{align}
    m_{t}(\Btheta) &\eqdef \Bk_{t}(\Btheta) \BK^{-1}_{t} \BDelta_{t} 
    + \BR_{t}\T(\Btheta) \bar{\Bgamma}_{t}, \label{eq:gp_mean} \\ 
    %
    \begin{split}
    c_{t}(\Btheta,\Btheta') 
    &\eqdef k(\Btheta,\Btheta') 
    - \Bk_{t}(\Btheta) \BK^{-1}_{t} \Bk\T_{t}(\Btheta') \\
    &+\BR_{t}\T(\Btheta)[\BB^{-1}+\BH_{t}\BK^{-1}_{t}\BH_{t}\T]^{-1} \BR_{t}(\Btheta'), 
    \end{split} \label{eq:gp_cov} 
    %
\end{align}
where $[\BK_{t}]_{ij} \eqdef k(\Btheta_{i},\Btheta_{j}) + \indic_{i=j}\sigma_n^2$, $\Bk_t(\Btheta) \eqdef (k(\Btheta,\Btheta_1),\ldots,k(\Btheta,\Btheta_t))\T$, $\BDelta_{t} \eqdef (\Delta_1,\ldots,\Delta_t)\T$ and
\begin{align}
    \bar{\Bgamma}_{t} &\eqdef [\BB^{-1}\!+\!\BH_{t}\BK^{-1}_{t}\BH_{t}\T]^{-1}\!(\BH_{t}\BK^{-1}_{t}\BDelta_{t}\!+\! \BB^{-1}\Bb), \\
    %
    &\BR_{t}(\Btheta) \eqdef \BH(\Btheta) - \BH_{t}\BK^{-1}_{t}\Bk_{t}\T(\Btheta).
\end{align}
Above $\bar{\Bgamma}_{t}$ is the generalised least-squares estimate, $\BH_{t}$ is the $r\times t$ matrix whose columns consist of basis function values evaluated at $\Btheta_{1:t}$, $\Btheta_{1:t}$ is a $p\times t$ matrix, and $\BH(\Btheta)\in\reals^r$ is the corresponding vector of test point $\Btheta$. 
We also define $s_{t}^2(\Btheta) \eqdef c_{t}(\Btheta,\Btheta)$ and $\probmeas_{D_t}^f \eqdef \GP(m_t(\Btheta),c_t(\Btheta, \Btheta'))$.
For further details of GP regression, see e.g.~\citet{Rasmussen2006}. 

If the true discrepancy mean function $f$ and the variance of the discrepancy $\sigma_n^2$ were known, the ABC posterior 
would be 
\begin{align}
\piabc^f(\Btheta) \eqdef \frac{\pi(\Btheta)\Phi\left({(\epsilon-f(\Btheta))}/{\sigma_n}\right)}
{\int_{\Theta}\pi(\Btheta')\Phi\left({(\epsilon-f(\Btheta'))}/{\sigma_n}\right) \ud \Btheta'}, 
\label{eq:mod_based_abc_post}
\end{align}
where $\Phi(\cdot)$ is the Gaussian cdf. This fact follows from (\ref{eq:abc_post}) and the Gaussian modelling assumption (\ref{eq:noisemodel}). 
In practice $f$ is unknown but our knowledge about $f$ is represented by the posterior $f \cond D_t \sim \probmeas_{D_t}^f$. Since the ABC posterior $\piabc^f$ in (\ref{eq:mod_based_abc_post}) depends on $f$, it is also a random quantity and its posterior can be obtained by propagating the uncertainty in $f$ through the mapping $f\mapsto\piabc^f$. 

Computing the distribution of $\piabc^f$ is difficult due to its nonlinear dependence on $f$ and because $f$ is infinite-dimensional. 
However, the pointwise mean, variance and quantiles of the unnormalised ABC posterior
%
%
\begin{align}
    \tildepiabc^f(\Btheta) \eqdef \pi(\Btheta)\Phi(({\epsilon-f(\Btheta)})/{\sigma_n}), \label{eq:unnorm_abc}
\end{align}
i.e.~the numerator of (\ref{eq:mod_based_abc_post}), can be computed analytically as shown by \citet{Jarvenpaa2018_acq} in the case of a zero mean GP prior. It is easy to see that their formulas also hold for our more general GP model in (\ref{eq:gp_prior}). For example, 
\begin{align}
\mean_{f\cond D_t}(\tildepiabc^f(\Btheta)) &= \pi(\Btheta)\Phi(a_t(\Btheta)), \label{eq:pimean} \\
a_t(\Btheta) &\eqdef \textstyle{({\epsilon \!-\! m_t(\Btheta)})/{\!\sqrt{\sigma_n^2\!+\!s^2_t(\Btheta)}}}, \label{eq:at} \\
\med_{f\cond D_t}(\tildepiabc^f(\Btheta)) &= \pi(\Btheta)\Phi\left({(\epsilon\!-\!m_t(\Btheta))}/{\sigma_n}\right),
\label{eq:pimed}
\end{align}
where $\med$ is the marginal (i.e.~elementwise) median. 
%
While these formulas are useful, they do not allow one to assess the uncertainty of e.g.~posterior mean $\int_{\Theta}\Btheta\,\piabc^f(\Btheta)\ud\Btheta$. We resolve this limitation in Section \ref{sec:post_uncertainty}.

\section{PARALLEL SIMULATIONS} \label{sec:parallel}

We aim to find the most informative simulation locations for obtaining the best possible estimate of the ABC posterior $\piabc^f$ given the postulated GP model. 
Principled sequential designs, where one simulation is run at a time, were developed by \citet{Jarvenpaa2018_acq}. In practice, to decrease the wall-time needed for the inference task, one could run some of the simulations in parallel. In the following, we apply Bayesian experimental design theory \citep{Chaloner1995,Jarvenpaa2019_sl} for the (synchronous) batch setting where $b$ simulations are simultaneously selected to be computed in parallel. 

\subsection{DECISION-THEORETIC APPROACH} \label{subsec:decision}

Consider a loss function $l:\Dset^2\rightarrow\reals_+$ so that $l(\piabc,d)$ quantifies the penalty of reporting $d\in\Dset$ as our ABC posterior when the true one is $\piabc\in\Dset$. 
Given $D_t$, the one-batch-ahead Bayes-optimal selection of the next batch of $b$ evaluations $\Btheta\opt = [\Btheta_1\opt,\ldots,\Btheta_b\opt]$ is then 
%
%
\begin{align}
    \Btheta\opt\!&=\!\argmin{\Btheta^*\in\Theta^b}L_t(\Btheta^*), \quad\text{where} \\
L_t(\Btheta^*)\!&=\!\mean_{\BDelta\!^*\!\cond\Btheta^*,D_t} \!\Big( 
\underbrace{\min_{d\in\Dset} \mean_{f\cond D_t\cup D^*} l(\piabc^f,d)}_{\eqdef \mathcal{L}(\probmeas_{D_t\cup D^*}^f)}\!\Big). 
\label{eq:cr}
\end{align}
%
In (\ref{eq:cr}), we calculate an expectation over future discrepancy evaluations $\BDelta\!^*=(\Delta^*_1,\ldots,\Delta^*_b)\T$ at locations $\Btheta^*$, assuming they follow our current GP model. The expectation is taken of the \emph{Bayes risk} $\mathcal{L}(\probmeas_{D_t\cup D^*}^f)$ resulting from the nested decision problem of choosing the estimator $d$, assuming $\BDelta\!^*$ are known and merged with current data $D_t$ via $D^* \eqdef \{(\Delta_i^*,\Btheta_i^*)\}_{i=1}^b$. 
While the main quantity of interest in the \babc{} framework is the ABC posterior $\piabc^f$ in (\ref{eq:mod_based_abc_post}), in practice it is desirable to use a loss function $\tilde{l}$ based on the unnormalised distribution $\tildepiabc^f$. Such a simplification, also used by \citet{Kandasamy2015,Sinsbeck2017,Jarvenpaa2018_acq,Jarvenpaa2019_sl}, allows efficient computations. Furthermore, evaluations that are optimal for estimating $\tildepiabc^f$ will be informative about the related quantity $\piabc^f$. 
%

Consider $L^2$ loss function $\tilde{l}_2 \eqdef \int_{\Theta}(\tildepiabc^f(\Btheta) - \tilde{d}(\Btheta))^2\ud\Btheta$ between the unnormalised ABC posterior $\tildepiabc^f$ and its estimator $\tilde{d}$ (both supposed to be square-integrable in $\Theta$ i.e.~$\tildepiabc^f,\tilde{d}\in L^2(\Theta)$). Then the optimal estimator is the mean in (\ref{eq:pimean}) \cite{Sinsbeck2017,Jarvenpaa2019_sl}. 
If we instead consider $L^1$ loss $\tilde{l}_1 \eqdef \int_{\Theta}|\tildepiabc^f(\Btheta) - \tilde{d}(\Btheta)|\ud\Btheta$ (supposing $\tildepiabc^f,\tilde{d}\in L^1(\Theta)$), then the marginal median in (\ref{eq:pimed}) is the optimal estimator. Corresponding Bayes risks, denoted $\mathcal{L}^{\textnormal{v}}$ and $\mathcal{L}^{\textnormal{m}}$, respectively, can be computed as follows: 
%
\begin{lemma} \label{lemma:imad}
Consider the GP model in Section \ref{sec:post_babc}. The Bayes risks for the $L^2$ and $L^1$ losses are given by
\begin{align}
\begin{split}
\mathcal{L}^{\textnormal{v}}(\probmeas_{D_t}^f) \!&=\! \int_{\Theta}\!\pi^2(\Btheta) \Big[ \Phi(a_t(\Btheta))\Phi(-a_t(\Btheta)) \\
%
%
&\quad\textstyle{-2T\Big(a_t(\Btheta), {\sigma_n/\!\sqrt{{\sigma_n^2\!+\!2s_t^2(\Btheta)}}}\Big)\Big]\!\ud\Btheta},
\end{split}
\label{eq:iv} \\
%
\mathcal{L}^{\textnormal{m}}(\probmeas_{D_t}^f) \!&=\! 2\!\int_{\Theta}\!\pi(\Btheta)T( a_t(\Btheta), {s_t(\Btheta)}/{\sigma_n} )\ud\Btheta, \label{eq:imad}
\end{align}
respectively, where $a_t(\Btheta)$ is given by (\ref{eq:at}) and $T(\cdot,\cdot)$ denotes \owen{} \cite{Owen1956}. 
\end{lemma}
%
We call $L_t(\Btheta^*)$ as an \emph{acquisition function}. Expected integrated variance (\eiv{}) and expected integrated MAD\footnote{Mean absolute deviation (around median).} (\eimad{}) acquisition functions, denoted $L_t^{\textnormal{v}}(\Btheta^*)$ and $L_t^{\textnormal{m}}(\Btheta^*)$, respectively, can be computed as follows: 
\begin{proposition} \label{prop:eimad}
Consider the GP model in Section \ref{sec:post_babc}.
The \eiv{} and \eimad{} acquisition functions are
\begin{align}
L_t^{\textnormal{v}}(\Btheta^*) \!&=\! 2\!\int_{\Theta}\!\pi^2(\Btheta)\!\Bigg[ T\!\left( \!a_t(\Btheta), \!\frac{\sqrt{\sigma_n^2\!+\!s_t^2(\Btheta) \!-\! \Deltav_t(\Btheta;\Btheta^*)}}{\sqrt{\sigma_n^2\!+\!s_t^2(\Btheta)\! +\! \Deltav_t(\Btheta;\Btheta^*)}}\right) \nonumber \\ 
%
&\myquad-T\!\left( \!a_t(\Btheta), \frac{\sigma_n}{\sqrt{\sigma_n^2\!+\!2s_t^2(\Btheta)}}\right)\!\Bigg]\!\ud\Btheta,
\label{eq:eiv}\\ 
%
L_t^{\textnormal{m}}(\Btheta^*) \!&=\! 2\!\int_{\Theta}\!\pi(\Btheta)T\!\left(\! a_t(\Btheta), \!{\frac{\sqrt{s^2_t(\Btheta) \!-\! \Deltav_t(\Btheta;\Btheta^*)}}{\sqrt{\sigma^2_n \!+\! \Deltav_t(\Btheta;\Btheta^*)}}} \right)\!\ud\Btheta, \label{eq:eimad}
\end{align}
respectively, where $a_t(\Btheta)$ is given by (\ref{eq:at}) and
\begin{align}
    \Deltav_{t}(\Btheta;\!\Btheta^*) 
\!=\! c_{t}(\Btheta,\!\Btheta^*)[c_{t}(\Btheta^*\!,\!\Btheta^*) \!+\! \sigma_n^2\Id]^{-1}c_{t}(\Btheta^*\!,\!\Btheta). \label{eq:gp_dvar}
\end{align}
\end{proposition}
This result generalizes \eiv{} in \citet{Jarvenpaa2018_acq} to the batch setting. The proofs are given in \appe{} \ref{appsubsec:proofs}.

\subsection{DETAILS ON COMPUTATION} \label{subsec:opt}

Finding the one-batch-ahead optimal design $\Btheta\opt$ requires global optimisation over $\Theta^b$ for both \eiv{} and \eimad{}. As this is infeasible with large batch size $b$ and/or the dimension $p$ of $\Btheta$, we use greedy optimisation: For $r=1,\ldots,b$, the $r$th point $\Btheta_r\opt$ in the batch is chosen by optimising $L_t([\Btheta^*_1,\ldots,\Btheta^*_r])$ with respect to $\Btheta^*_r$ when the earlier points $\Btheta^*_1,\ldots,\Btheta^*_{r-1}$ are kept fixed to their already determined values. This simplifies the $pb$-dimensional optimisation problem to a sequence of easier $p$-dimensional problems. Similar techniques have been used in batch BO, see~\citet{Ginsbourger2010,Snoek2012,Wilson2018}. 
%
Theory of submodular optimisation has been used to study greedy batch designs \cite{Bach2013,Wilson2018,Jarvenpaa2019_sl}. Unfortunately, such analysis hardly extends to our case because the acquisition functions in Proposition \ref{prop:eimad} depend on $\Btheta^*$ in a rather complex way. 
Using the facts that $T(h,a)$ is non-decreasing for $a\geq 0$ 
and $\Deltav_t(\Btheta;\Btheta^*)$ cannot decrease as more points are included to $\Btheta^*$, we nevertheless see that both \eiv{} and \eimad{} are non-increasing as set functions of $\Btheta^*$. We can thus expect the greedy optimisation to be useful in practice as is seen empirically in Section \ref{sec:experiments}. 

Another potential computational difficulty is the integration over $\Theta$ in (\ref{eq:eiv}) and (\ref{eq:eimad}). Many state-of-the-art BO methods, such as \citet{Hennig2012,HernandezLobato2014,Wu2016}, also require similar computations. We approximate the integral using numerical integration for $p\leq 2$ and self-normalised importance sampling (IS), where the current loss function interpreted as an unnormalised density is the instrumental distribution, for $p>2$. 
Full details and the pseudocode of our algorithm can be found in \appe{} \ref{app:implementation}.

\subsection{HEURISTIC BASELINE BATCH METHODS} \label{subsec:heuristic}

We consider also heuristic acquisition functions which evaluate where the pointwise uncertainty of $\tildepiabc^f(\Btheta)$ is highest. Such intuitive strategies are also known as \emph{uncertainty sampling} and used e.g.~by \citet{Gunter2014,Jarvenpaa2018_acq,Chai2019}. When the variance is used as the measure of the uncertainty of $\tildepiabc^f(\Btheta)$, we call the method as \maxv{}. When MAD is used, we obtain an alternative strategy called analogously \maxmad{}.  
The resulting acquisition functions can be computed using the integrands of (\ref{eq:iv}) and (\ref{eq:imad}).

Finally, we propose a heuristic approach, also used for batch BO \cite{Snoek2012}, to parallellise \maxv{} and \maxmad{} strategies: The first point in the batch is chosen as in the sequential case. The other points are iteratively selected as the locations where the expected pointwise variance (or MAD) of $\tildepiabc^f(\Btheta)$, taken with respect to the discrepancy values of the pending points (i.e.~points that have been already chosen to the current batch) is highest. The resulting acquisition functions are immediately obtained as the integrands of (\ref{eq:eiv}) and (\ref{eq:eimad}). 

\section{UNCERTAINTY QUANTIFICATION OF THE ABC POSTERIOR} \label{sec:post_uncertainty}

Pointwise marginal uncertainty of the unnormalised ABC posterior $\tildepiabc^f$ was used in previous section for selecting the simulation locations adaptively. However, knowing the value of $\tildepiabc^f$ and its marginal uncertainty in some individual $\Btheta$-values is not very helpful for summarising and understanding the accuracy of the final estimate of the ABC posterior. 
Computing the distribution of the moments and marginals of the normalised ABC posterior $\piabc^f$ in (\ref{eq:mod_based_abc_post}) is clearly more intuitive. See Fig.~\ref{fig:post_uncertainty} for a 1D demonstration of this approach. 

\begin{figure}[hbt!] 
\centering
\includegraphics[width=0.42\textwidth]{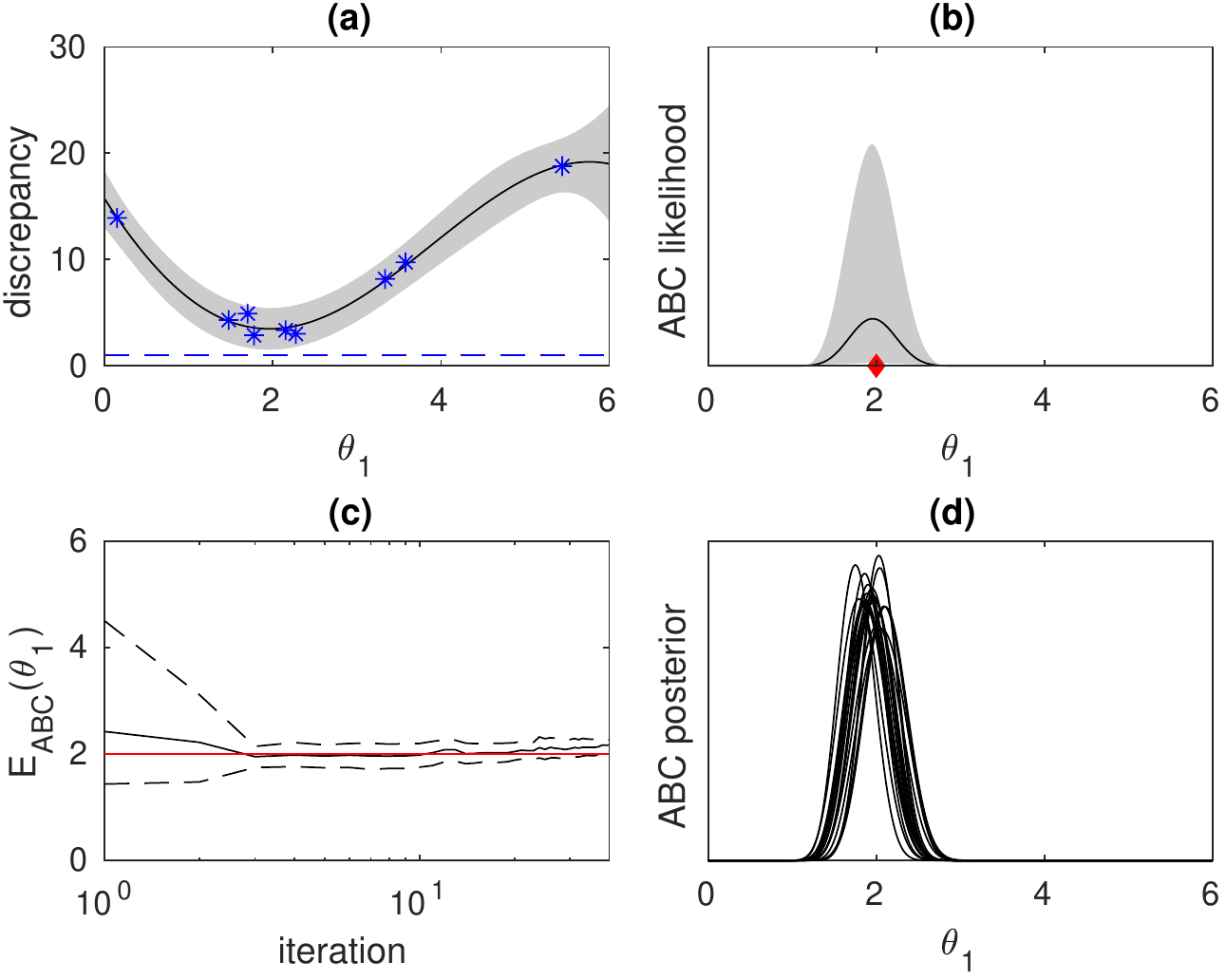}
\caption{Demonstration of ABC posterior uncertainty quantification using Lorenz model from Section \ref{subsec:simmodels} with parameter $\theta_2$ fixed. (a) GP model for $\Delta_{\theta_1}$ (blue dashed line $\epsilon$, blue stars $9$ discrepancy evaluations), (b) uncertainty of unnormalised ABC posterior $\tildepiabc^f$, (c) evolution of model-based ABC posterior expectation (black line) and its 95\% CI (dashed black) for $40$ iterations, (d) uncertainty of ABC posterior $\piabc^f$ corresponding (b).} \label{fig:post_uncertainty}
\end{figure}

To access the posterior of $\piabc^f$, one could fix a sample path $f^{(i)}\sim\probmeas_{D_t}^f$, then use it to fix a realisation of the ABC posterior $\piabc^{f^{(i)}}$ using (\ref{eq:mod_based_abc_post}) and finally use e.g.~MCMC to sample from $\piabc^{f^{(i)}}$. This would be repeated $s$ times and the resulting set of samples $\{\{\Btheta^{(i,j)}\}_{j=1}^{n}\}_{i=1}^{s}$ (where $n$ is the length of the MCMC chain for each posterior realisation $i=1,\ldots,s$) approximately describes the posterior of $\piabc^f$ given $D_t$ (see Fig.~\ref{fig:post_uncertainty}d). 
The uncertainty of the GP hyperparameters $\Bpsi$ could also be taken into account by drawing $\Bpsi^{(i)}\sim \pi(\Bpsi\cond D_t)$ as the very first step but we here consider $\Bpsi$ as known for simplicity although this can cause underestimation of the uncertainty of $\piabc^f$.
The outlined approach involves a major computational challenge as evaluating the $s$ sample paths at $n$ distinct sets of test points scales\footnote{Approximations such as random Fourier features (RFF) \cite{Rahimi2008rff} and those by \citet{Pleiss2018} can be used to reduce this cost, e.g.~\citet{HernandezLobato2014,Wang2017} used RFF to approximately optimise GP sample paths. However, this produces tradeoff between exact GP but small $n$ vs.~inexact GP but large $n$ which we do not analyse in this work. 
} as $\bigO(s(nt^2 + tn^2)+sn^3)$.
%

We propose the following computationally cheaper 
approach: 
In small dimensions, when $p\leq 2$, we evaluate each sample path $f^{(i)},i=1,\ldots,s$ at $\bar{n}^p$ fixed grid points and compute the required integrations numerically. This approach scales as $\bigO(\bar{n}^pt^2 + \bar{n}^{2p}(t+s) + \bar{n}^{3p})$.
If $p>2$, then self-normalised importance sampling is used. We draw $n$ samples from an instrumental density, defined so that its unnormalised pdf at $\Btheta$ equals the $\alpha$-quantile of $\tildepiabc^f(\Btheta)$. This is computed using (\ref{eq:quantile}) of \appe{} and we use $\alpha= 0.95$. The samples are thinned and the resulting $\tilde{n}\ll n$ representative samples $\{\Btheta^{(j)}\}_{j=1}^{\tilde{n}}$ are used to compute the normalised importance weights $\omega^{(i,j)}$ for each sampled posterior $i=1,\ldots,s$. The output is a set of weighted sample sets $\{\{(\omega^{(i,j)},\Btheta^{(j)})\}_{j=1}^{\tilde{n}}\}_{i=1}^{s}$ from which moments and marginal densities can be computed using standard Monte Carlo estimators for each $i=1,\ldots,s$. This approach requires only one MCMC sampling from the instrumental density which scales as $\bigO(nt^2)$, i.e.~only linearly with respect to $n$, so that $n$ can be large. Total cost is $\bigO((n+\tilde{n})t^2 + \tilde{n}^2(t+s) + \tilde{n}^3)$. 
%

This approach has nevertheless some limitations: The computations are only approximate because $\tilde{n}$ and $s$ are finite. 
Also, if the uncertainty of $\piabc^f$ is substantial, choosing a good instrumental density can be difficult. This is because some of the sampled posteriors are then necessarily quite different from any single instrumental density producing possibly poor approximation. In our experiments this however happened only with early iterations and can be detected e.g.~by monitoring the distribution of effective sample sizes for $i=1,\ldots,s$. 
In Section \ref{sec:experiments} we demonstrate that the uncertainty quantification is still feasible and beneficial for low-dimensional cases. 
The proposed approach also works with other GP modelling situations such as \citet{Jarvenpaa2019_sl}.

\section{ON RELATED GP-BASED METHODS} \label{sec:theory}

In this section we briefly discuss the relation between \babc{}, BQ and BO to facilitate better understanding of these conceptually similar inference methods.

\subsection{RELATION TO BAYESIAN QUADRATURE} \label{subsec:bq}

In Bayesian quadrature one aims to compute integral $I_f \eqdef \int_{\reals^p}f(\Btheta)\pi(\Btheta)\ud \Btheta$, where $f:\reals^p\rightarrow\reals$ is an expensive black-box function and $\pi(\Btheta)$ is a known density, e.g.~Gaussian. If a GP prior is placed on $f$, given some evaluations $\{(f_i,\Btheta_i)\}_{i=1}^t$ where $f_i=f(\Btheta_i)$, the posterior of $I_f$, describing one's knowledge of the value of this integral, is Gaussian whose mean and variance can be computed analytically for some choices of $k(\Btheta,\Btheta')$ and $\pi(\Btheta)$ (for details, see \citet{OHagan1991,Briol2019}). 
Also, BQ methods for computing integrals of the form $I_f^g \eqdef \int_{\Theta}g(f(\Btheta))\pi(\Btheta)\ud\Btheta$ with some known (non-negative) function $g:\reals\rightarrow\reals_+$, such as marginal likelihoods, have been developed by \citet{Osborne2012,Gunter2014,Chai2019}. 

Our approach in Section \ref{sec:post_uncertainty} is instead developed for quantifying the uncertainty in either the whole function $\piabc^f:\Theta\rightarrow\reals_+$, which we here write as 
\begin{align}
\piabc^f(\Btheta) = \frac{g(f(\Btheta)) \pi(\Btheta)}
{\int_{\Theta} g(f(\Btheta')) \pi(\Btheta') \ud \Btheta'}, \label{eq:mod_based_gen_post}
\end{align}
or some corresponding moments such as the expectation $\int_{\Theta}\Btheta\,\piabc^f(\Btheta)\ud\Btheta\in\reals^p$. 
To our knowledge, computation of these quantities probabilistically has not been considered before. 
In particular, we used the ``0-1 kernel'' $\indic_{\Delta_{\Btheta}\leq\epsilon}$ in (\ref{eq:abc_post}) corresponding to $g(f(\Btheta)) = \Phi((\epsilon-f(\Btheta))/\sigma_n)$ in (\ref{eq:mod_based_gen_post}). 
\citet{Osborne2012,Gutmann2016,Acerbi2018,Jarvenpaa2019_sl} instead modelled the log-likelihood with GP to reckon the non-negativity of the likelihood and the high dynamic range of the log-likelihood. This would correspond to $g(f(\Btheta)) = \exp(f(\Btheta))$ in (\ref{eq:mod_based_gen_post}). 


\citet{Osborne2012,Gunter2014,Chai2019} used linearisation approximations in their algorithms for estimating integrals of the form $I_f^g$. 
Similarly, if both $\Phi(\cdot)$-terms in our case in (\ref{eq:mod_based_abc_post}) were linear for $f$, then the numerator and denominator in (\ref{eq:mod_based_abc_post}) would have joint Gaussian density leading to tractable computations. However, we observed that the resulting densities can be highly non-Gaussian so that any linearisation approach can result poor quality approximations. 
For this reason we considered simulation-based approach in Section \ref{sec:post_uncertainty}.

\subsection{RELATION TO BAYESIAN OPTIMISATION} \label{subsec:bo}

Suppose now $f:\Theta\subset\reals^p\rightarrow\reals$ is an expensive, black-box function to be minimised. In BO, a GP prior is placed on $f$ and the future locations for obtaining (possibly noisy) evaluations of $f$ are chosen adaptively by optimising an acquisition function that, in some sense, measures the potential improvement in the knowledge of the minimum point $\Btheta^{\star} \eqdef \arg\min_{\Btheta\in\Theta}f(\Btheta)$ or the corresponding function value $f^{\star} \eqdef \min_{\Btheta\in\Theta}f(\Btheta)$ brought by the extra evaluation. 
%
For example, (predictive) entropy search \cite{Hennig2012,HernandezLobato2014} use an acquisition function that measures the expected reduction in the differential entropy of the posterior of $\Btheta^{\star}$. \citet{Wang2017} similarly considered the posterior of $f^{\star}$. 
The important difference between these methods (or BO in general) and \babc{} is that the quantity of interest in \babc{} is not the minimiser of $f$ but the full ABC posterior density $\piabc^f$ (or $\tildepiabc^f$). 
Also, BO is rarely introduced this way in literature, simple acquisition functions such as the expected improvement and lower confidence bound (LCB) are often used and the posterior of $\Btheta^{\star}$ or $f^{\star}$ is rarely considered. 

In the \emph{BOLFI} framework \cite{Gutmann2016}, the function $f$ was however taken to be the ABC discrepancy $\Delta_{\Btheta}$, and LCB acquisition function $\textnormal{LCB}(\Btheta) = m_t(\Btheta) - \beta_t s_t(\Btheta)$ \cite{Srinivas2010} was used for illustrating their approach of learning the ABC posterior. This is reasonable because to learn the ABC posterior one needs to evaluate in the regions with small discrepancy. 
We have the following new result that relates LCB to the \babc{} framework:
\begin{proposition} \label{prop:lcbconnection}
If the prior is uniform over $\Theta$ (and may be improper), i.e.~if $\pi(\Btheta)\propto \indic_{\Btheta\in\Theta}$, then the point chosen by the LCB acquisition function with parameter $\beta_t$ is the same as the point maximising the $\Phi(\beta_t)$-quantile of the unnormalised ABC posterior $\tildepiabc^f(\Btheta)$ for any $\epsilon$. 
\end{proposition}
This result gives an interpretation for the LCB tradeoff parameter $\beta_t$ in the ABC setting. 
However, instead of using LCB for \babc{}, it is clearly more reasonable to evaluate where the variance (or some other measure of uncertainty) is large as already discussed e.g.~by \citet{Kandasamy2015,Jarvenpaa2018_acq}. \citet{Jarvenpaa2018_acq} showed empirically that \eiv{} consistently works better than LCB in their sequential scenario when the goal is to learn the ABC posterior. 
For this reason, we do not use (batch) BO methods in this article.

\section{EXPERIMENTS} \label{sec:experiments}

\begin{figure*}[htbp!] 
\centering
\includegraphics[width=0.9\textwidth]{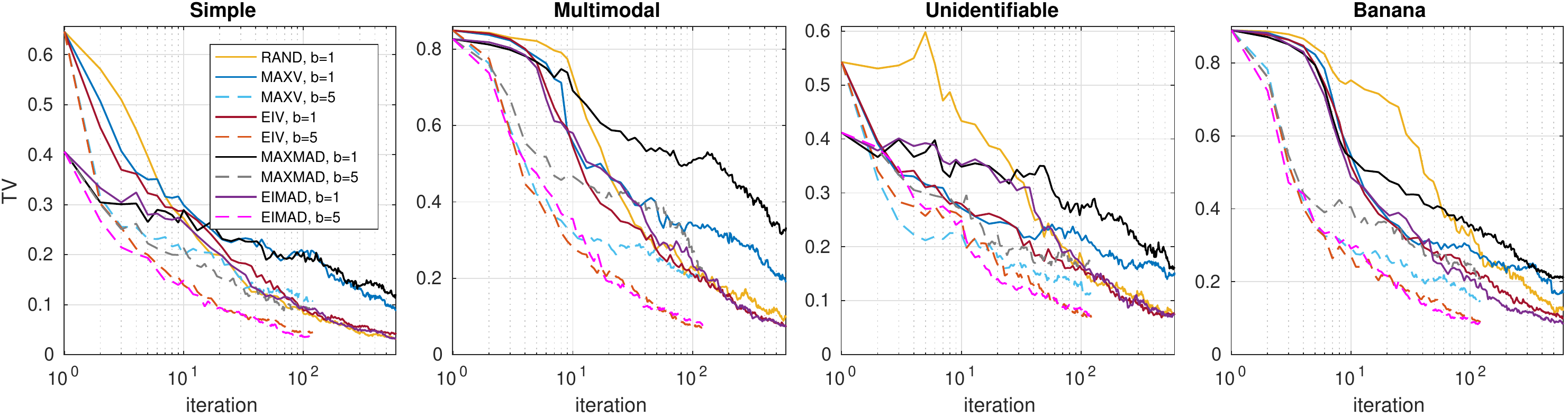}
\caption{Results for the 2D toy simulation models over $600$ iterations and two batch sizes $b$.} \label{fig:synth1}
\end{figure*}

\begin{figure*}[htbp!] 
\centering
\includegraphics[width=0.9\textwidth]{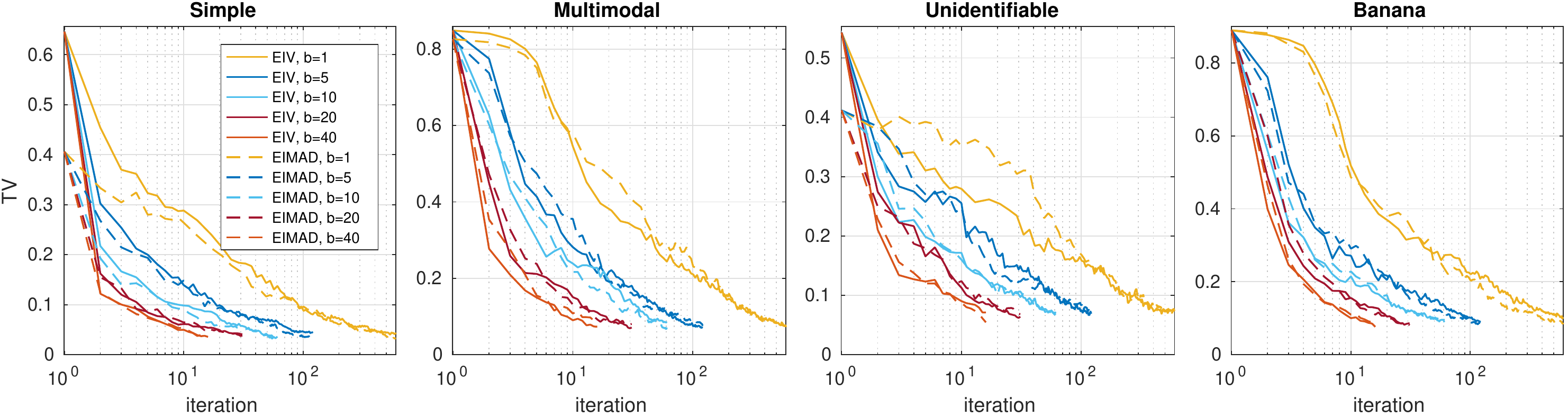}
\caption{Results for 2D toy simulation models with two acquisition functions and various batch sizes.} \label{fig:synth_batch}
\end{figure*}

We first consider four 2D toy problems to see how the proposed method performs with a well-specified GP model. We then focus on more typical scenarios where the GP modelling assumptions do not hold exactly using three real-world simulation models. We compare the performance of the sequential and synchronous batch versions of the acquisition methods of Section \ref{sec:parallel}. As a simple baseline, we consider random points drawn from the prior (abbreviated as RAND).
%
%
We also briefly demonstrate the uncertainty quantification of the ABC posterior. 
We do not consider 
synthetic likelihood method (as e.g.~in \citet{Jarvenpaa2019_sl}) because it requires hundreds of evaluations for each proposed parameter and is thus not applicable here. 
For similar reason, we do not consider sampling-based ABC methods. 

Locations for fitting the initial GP model are sampled from the uniform prior in all cases. We take $10$ initial points for 2D and $20$ for 3D and 4D cases.  
We use $\Bb=\Bzeros$, $B_{ij}=10^2\indic_{i=j}$ and include basis functions of the form $1, \theta_i, \theta_i^2$. 
The discrepancy $\Delta_{\Btheta}$ is assumed smooth and we use the squared exponential covariance function $k(\Btheta,\Btheta') = \sigma_{\lik}^2 \exp(-\half\sum_{i=1}^p(\theta_i-\theta_i')^2/l_i^2)$. 
GP hyperparameters $\Bpsi=(\sigma_n^2,l_1,\ldots,l_p,\sigma_f^2)$ are given weakly informative priors and their values are obtained using MAP estimation at each iteration. 

ABC-MCMC \cite{Marjoram2003} with extensive simulations is used to compute the ground truth ABC posterior for the real-world models. 
For simplicity and to ensure meaningful comparisons to ground-truth, we fix $\epsilon$ to certain small predefined values although, in practice, its value is set adaptively \cite{Jarvenpaa2018_acq} or based on pilot runs. 
We compute the estimate of the unnormalised ABC posterior using (\ref{eq:pimean}) for \maxv{}, \eiv{}, RAND and (\ref{eq:pimed}) for \maxmad{}, \eimad{}. Adaptive MCMC is used to sample from the resulting ABC posterior estimates and from the instrumental densities needed for the IS approximations. 
TV denotes the median total variation distance between the estimated ABC posterior and the true one (2D) or the average TV between their marginal TV values (3D, 4D) computed numerically over $50$ repeated runs. 
Iteration (i.e.~number of batches chosen) serves as a proxy to wall-time. The number of simulations i.e.~the maximum value of $t$ is fixed in all experiments and the batch methods thus finish earlier. 


\subsection{TOY SIMULATION MODELS} \label{subsec:synth}

Fig.~\ref{fig:synth1} shows the results with sequential methods ($b=1$) and the corresponding batch methods with $b=5$ for four synthetically constructed toy models. These were taken from \citet{Jarvenpaa2018_acq} and are illustrated in the \appe{} \ref{app:implementation}. In Fig.~\ref{fig:synth_batch} the effect of batch size  $b$ is studied for the two best performing methods.

\subsection{REAL-WORLD SIMULATION MODELS} \label{subsec:simmodels}

\textbf{Lorenz model.} This modified version of the well-known Lorenz weather prediction model describes the dynamics of slow weather variables and their dependence on unobserved fast weather variables over a certain period of time. The model is represented by a coupled stochastic differential equation which can only be solved numerically resulting in an intractable likelihood function. The model has two parameters $\Btheta=(\theta_1,\theta_2)$ which we estimate from timeseries data generated using $\Btheta=(2,0.1)$. See \citet{Thomas2018} for full details of the model and the experimental set-up that we also use here, with the exception that we use wider uniform prior $\Btheta\sim\Unif([0,5]\!\times\![0,0.5])$. The discrepancy is formed as a Mahalanobis distance from the six summary statistics by \citet{Hakkarainen2012}. 
The results are shown in Fig.~\ref{fig:lorenz1}(a). 
Furthermore, Fig.~\ref{fig:lorenz1}(b-c) demonstrates the uncertainty quantification of the model-based ABC posterior expectation. 
See \appe{} \ref{app:impl} for the details of the numerical computations used.
The effect of batch size is shown in Fig.~\ref{fig:bact}(c). 

\begin{figure*}[hbtp!] 
\centering
\begin{subfigure}{0.3\textwidth}
\includegraphics[width=\textwidth]{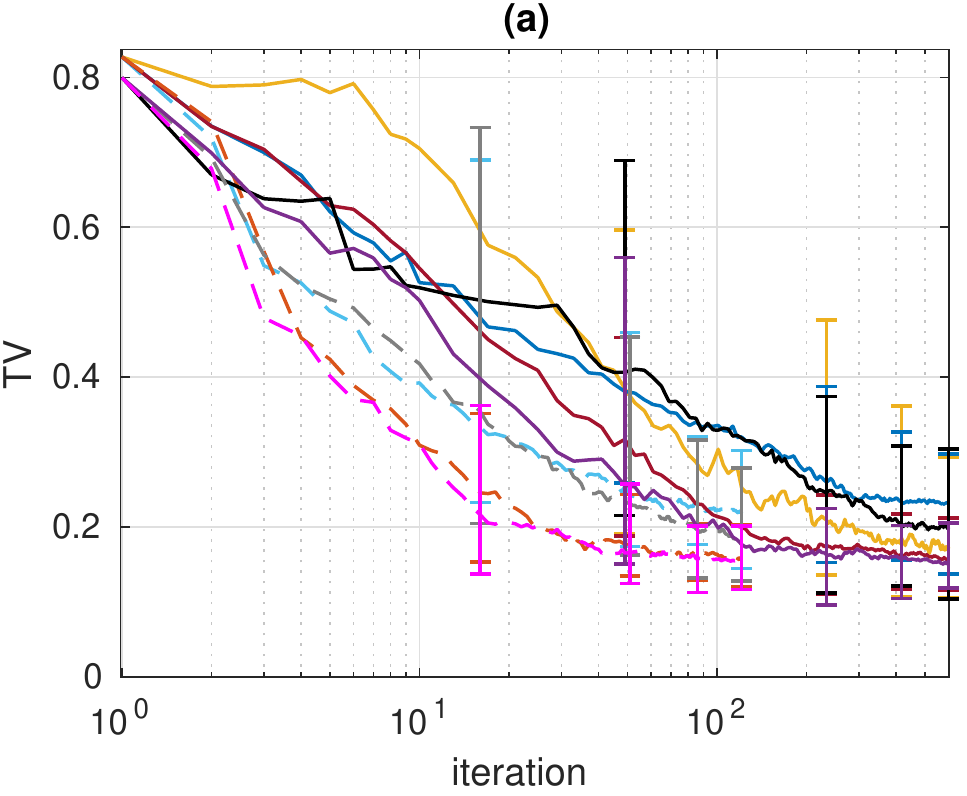}
\end{subfigure}
\hspace{0.17cm}
\begin{subfigure}{0.62\textwidth}
\includegraphics[width=\textwidth]{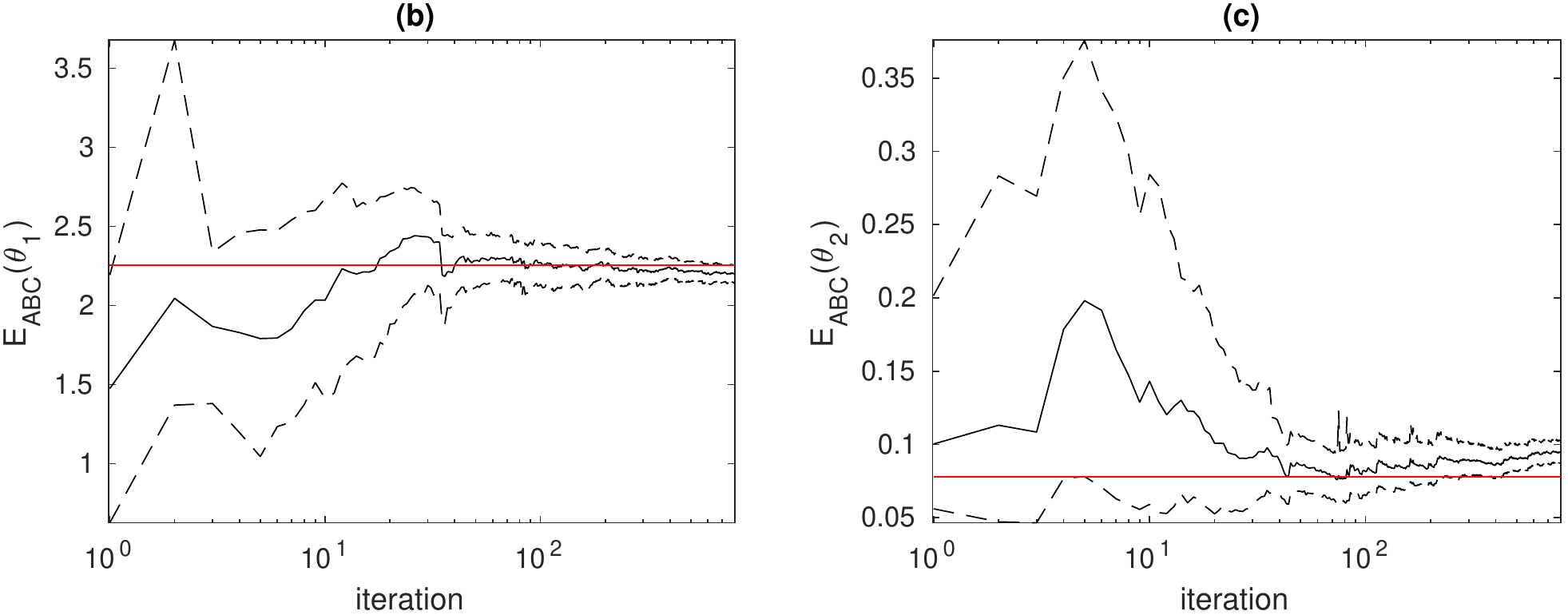}
\end{subfigure}
\caption{(a) Lorenz model. The intervals show the $90\%$ variability. See Fig.~\ref{fig:synth1} for the legend. (b-c) Black line is the mean and dashed black the $95$\% CI of the ABC posterior expectations. Red line shows the true value. } \label{fig:lorenz1}
\end{figure*}

\begin{figure*}[hbtp!] 
\centering
\begin{subfigure}{0.29\textwidth}
\includegraphics[width=\textwidth]{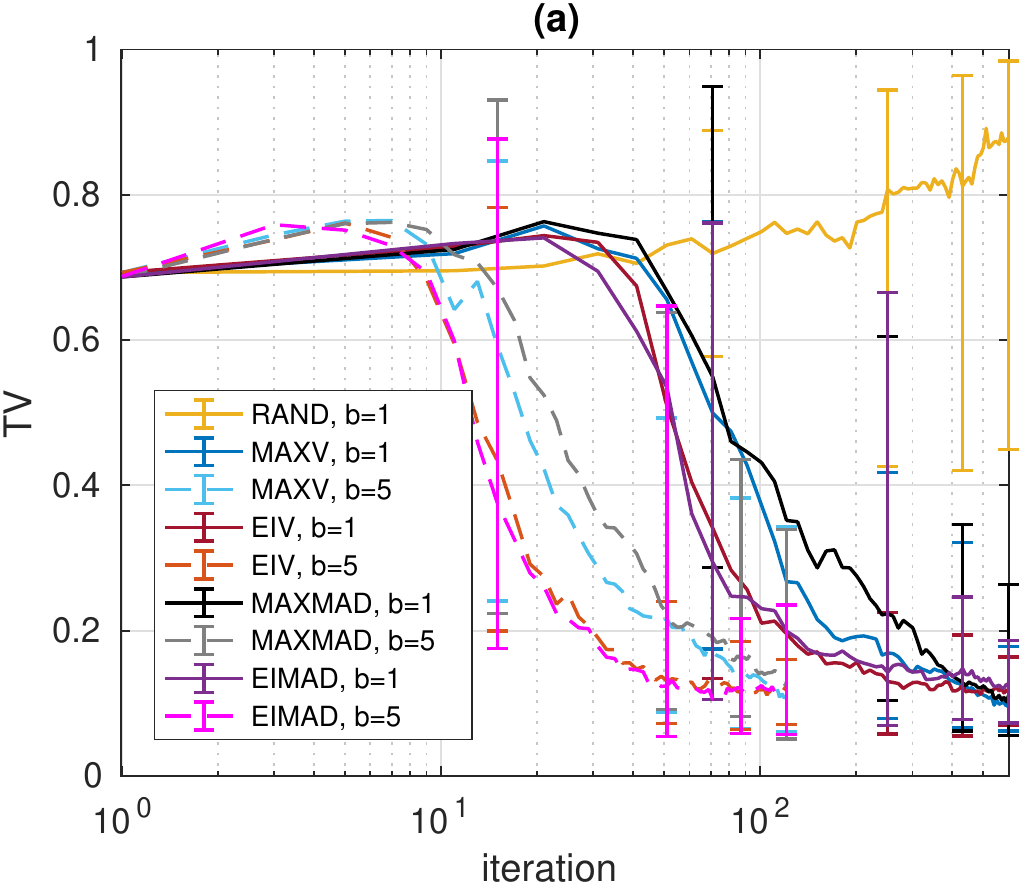}
\end{subfigure}
\begin{subfigure}{0.29\textwidth}
\includegraphics[width=\textwidth]{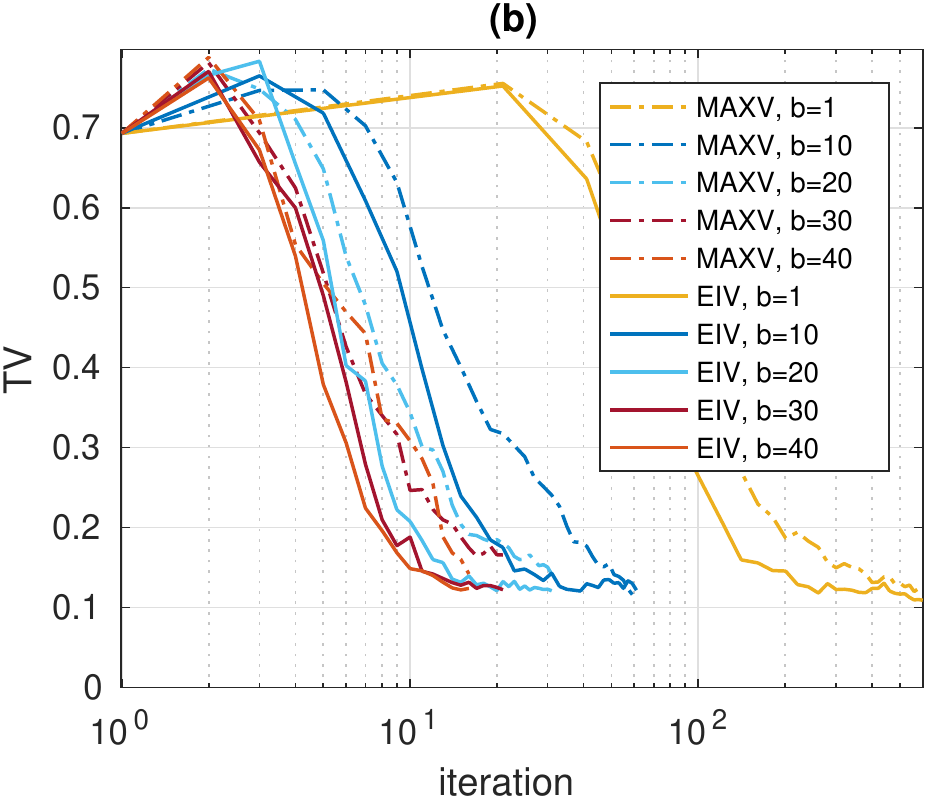}
\end{subfigure}
\hspace{0.05cm}
\begin{subfigure}{0.29\textwidth}
\includegraphics[width=\textwidth]{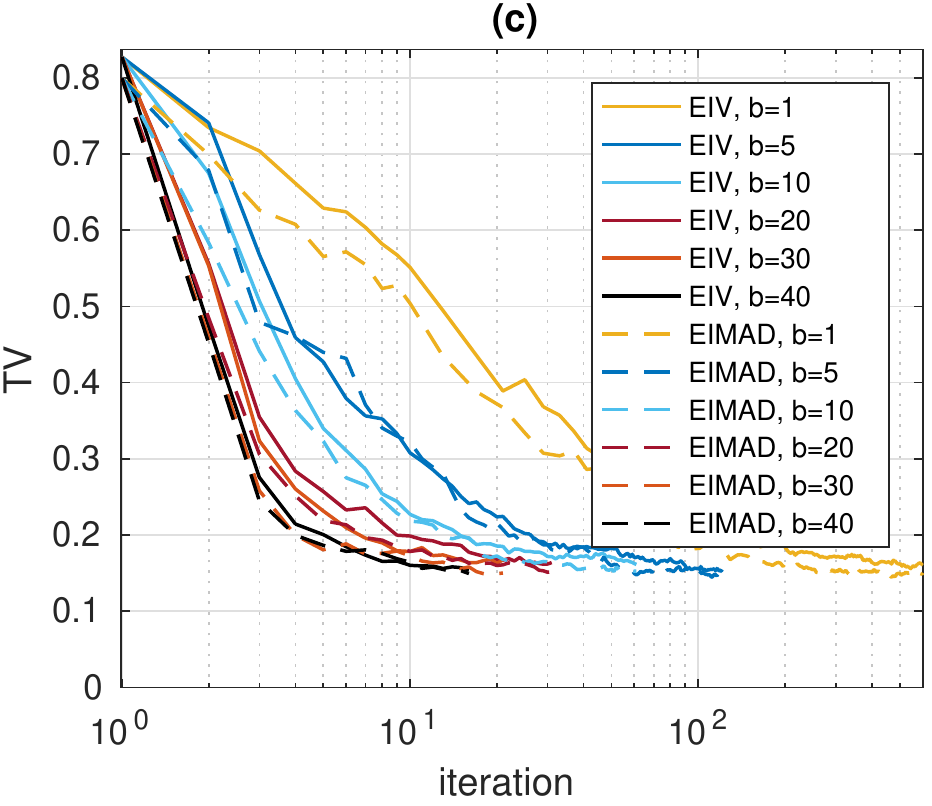}
\end{subfigure}
\caption{(a) Bacterial infections model. The intervals show the $90\%$ variability. (b) Bacterial infections model with different batch sizes and two chosen acquisition methods. (c) Additional experiments with Lorenz model.} \label{fig:bact}
\end{figure*}

\textbf{Bacterial infections model.} This model describes transmission dynamics of bacterial infections in day care centres and features intractable likelihood. The model has been developed by \citet{Numminen2013} and used previously by \citet{Gutmann2016,Jarvenpaa2018_acq} as an ABC benchmark problem. We estimate the internal, external and co-infection parameters $\beta\in[0,11],\Lambda\in[0,2]$ and $\theta\in[0,1]$, respectively, using true data \cite{Numminen2013} and uniform priors. The discrepancy is formed as in \citet{Gutmann2016}, see \appe{} \ref{app:add_exp} for details. 
The results with all methods are shown in Fig.~\ref{fig:bact}(a) and Fig.~\ref{fig:bact}(b) shows the effect of batch size for the two best performing methods.

Additional details, e.g.,~on the optimisation of the acquisition function, MCMC methods used, computational costs, and additional experimental results can be found in the \appe{} \ref{app:implementation} and \ref{app:experiments}. The results for our third, additional ABC benchmark scenario, \textbf{g-and-k model}, are shown in the \appe{} \ref{app:gk}.


\subsection{DISCUSSION ON THE RESULTS} \label{subsec:remarks}

In general, we obtain reasonable posterior approximations considering the very limited budget of simulations. 
\eiv{} and \eimad{} tend to produce more stable, accurate but also more conservative estimates than \maxv{} and \maxmad{}. 
%
Difference in approximation quality between \eiv{} and \eimad{}, both based on the same Bayesian decision theoretic framework but different loss functions, was small. 
%
While RAND worked well in 2D cases and is fully parallellisable, it unsurprisingly produced poor posterior approximations in higher dimensions. 
In all cases, our batch strategies produced similar evaluation locations as the corresponding sequential methods leading to substantial improvements in wall-time when the simulations are costly. 
%
Unlike in the related problem of BO, batch points need not always be diverse because the simulations are stochastic and simulating multiple times at nearby points can be useful. On the other hand, already a single simulation can be enough to effectively rule out large tail regions. The proposed methods automatically balance between these two situations.

Fig.~\ref{fig:lorenz1}(b-c) shows the evolution of the uncertainty in the ABC posterior expectation of the Lorenz model over $800$ iterations in the case of sequential \eiv{}. The convergence is approximately towards the true ABC posterior expectation due to a slight GP misspecification. 
Similarly, the ABC posterior marginals of the bacterial infection model in \appe{} \ref{app:experiments} contain some uncertainty after $600$ iterations which our approach allows to rigorously quantify. 
%
Due to the approximations involved and because this approach is not designed to account for the error due to approximating the intractable ground-truth posterior with the ABC posterior in the first place, we however suggest to interpret the uncertainty estimates with care. 
Developing more effective (analytical) methods for computing these uncertainty estimates is an interesting avenue for future work. The connection to BQ methods outlined in Section \ref{subsec:bq} can be helpful for achieving this goal.

\section{CONCLUSIONS} \label{sec:concl}

We considered ABC inference with a limited number of simulations ($t\!\lesssim\!1000$). We outlined a GP surrogate modelling framework called \babc{} where the uncertainty of the ABC posterior distribution due to the limited computational resources is approximately quantified. 
We also developed batch-sequential Bayesian experimental design strategies to efficiently parallellise the expensive simulations. Experiments suggest that substantial gains in wall-time over previous related work can be obtained. 
%


\subsubsection*{Acknowledgements}

We thank Academy of Finland (grants no.~286607 and 294015 to PM) and Finnish Center for Artificial Intelligence for support of this research. We acknowledge the computational resources provided by Aalto Science-IT project.


\bibliography{refs_final}
\bibliographystyle{plainnat}

\onecolumn
\appendix
\numberwithin{equation}{section}
\numberwithin{figure}{section}


\section{Proofs and additional analysis} \label{app:proofs}

\subsection{Proofs} \label{appsubsec:proofs}

\begin{proof}[Proof of Lemma \ref{lemma:imad}]
We consider the case of integrated variance first. A result corresponding to (\ref{eq:iv}) but with zero mean GP prior is shown as Lemma 3.1 in the article by \citet{Jarvenpaa2018_acq}. However, its proof works as such also for our GP model in Section \ref{sec:post_babc} and (\ref{eq:iv}) follows immediately.

Let us now consider integrated MAD in (\ref{eq:imad}). 
To simplify notation, we use $m_{\Btheta}$ for $m_t(\Btheta)$,  $s^2_{\Btheta}$ for $s^2_t(\Btheta)$ and $f_{\Btheta}$ for $f(\Btheta)$. We then see that
\begin{align}
    &\mean_{f\cond D_t} \int_{\Theta} \left|\pi(\Btheta)\Phi\left( \frac{\epsilon - f_{\Btheta}}{\sigma_n}\right) - \pi(\Btheta)\Phi\left( \frac{\epsilon - m_{\Btheta}}{\sigma_n}\right)\right| \ud \Btheta \\
    &= \int_{\Theta} \pi(\Btheta) \int_{-\infty}^{\infty} \left| \Phi\left( \frac{\epsilon - f_{\Btheta}}{\sigma_n}\right) - \Phi\left( \frac{\epsilon - m_{\Btheta}}{\sigma_n} \right) \right| \Normal(f_{\Btheta}\cond m_{\Btheta},s^2_{\Btheta}) \ud f_{\Btheta} \ud \Btheta.
\end{align}
For the inner integral with fixed $\Btheta$ we obtain
\begin{align}
    &\int_{-\infty}^{\infty} \left| \Phi\left( \frac{\epsilon - f_{\Btheta}}{\sigma_n}\right) - \Phi\left( \frac{\epsilon - m_{\Btheta}}{\sigma_n} \right) \right| \Normal(f_{\Btheta}\cond m_{\Btheta},s^2_{\Btheta}) \ud f_{\Btheta} \\
    \begin{split}
    &= \int_{-\infty}^{m_{\Btheta}} \left[ \Phi\left( \frac{\epsilon - f_{\Btheta}}{\sigma_n}\right) - \Phi\left( \frac{\epsilon - m_{\Btheta}}{\sigma_n} \right) \right] \Normal(f_{\Btheta}\cond m_{\Btheta},s^2_{\Btheta}) \ud f_{\Btheta} \\
    &\myquad+ \int_{m_{\Btheta}}^{\infty} \left[ \Phi\left( \frac{\epsilon - m_{\Btheta}}{\sigma_n}\right) - \Phi\left( \frac{\epsilon - f_{\Btheta}}{\sigma_n} \right) \right] \Normal(f_{\Btheta}\cond m_{\Btheta},s^2_{\Btheta}) \ud f_{\Btheta} \end{split} \\
    &= \int_{-\infty}^{m_{\Btheta}} \Phi\left( \frac{\epsilon - f_{\Btheta}}{\sigma_n}\right) \Normal(f_{\Btheta}\cond m_{\Btheta},s^2_{\Btheta}) \ud f_{\Btheta}
    - \int_{m_{\Btheta}}^{\infty} \Phi\left( \frac{\epsilon - f_{\Btheta}}{\sigma_n}\right) \Normal(f_{\Btheta}\cond m_{\Btheta},s^2_{\Btheta}) \ud f_{\Btheta} \\
    &= 2\int_{-\infty}^{m_{\Btheta}} \Phi\left( \frac{\epsilon - f_{\Btheta}}{\sigma_n}\right) \Normal(f_{\Btheta}\cond m_{\Btheta},s^2_{\Btheta}) \ud f_{\Btheta}
    - \Phi\left( \frac{\epsilon - m_{\Btheta}}{\sqrt{\sigma_n^2+s^2_{\Btheta}}}\right),
\end{align}
where on the last line we have used the fact
\begin{align}
    \int_{-\infty}^{\infty}\Phi\left( \frac{\epsilon - f_{\Btheta}}{\sigma_n}\right)\Normal(f_{\Btheta}\cond m_{\Btheta},s^2_{\Btheta}) \ud f_{\Btheta} = \Phi\left( \frac{\epsilon - m_{\Btheta}}{\sqrt{\sigma_n^2+s^2_{\Btheta}}}\right) \label{eq:helpformula1}
\end{align}
shown by \citet{Jarvenpaa2018_acq}. We further see that
\begin{align}
&\int_{-\infty}^{m_{\Btheta}} \Phi\left( \frac{\epsilon - f_{\Btheta}}{\sigma_n}\right) \Normal(f_{\Btheta}\cond m_{\Btheta},s^2_{\Btheta}) \ud f_{\Btheta} \\
&= \int_{-\infty}^{0} \Phi\left( \frac{\epsilon - m_{\Btheta} - y}{\sigma_n}\right) \Normal(y\cond 0,s^2_{\Btheta}) \ud y \quad [\textnormal{transformation}\;y = f_{\Btheta} - m_{\Btheta}] \\
&= \int_{-\infty}^{0} \int_{-\infty}^{\epsilon - m_{\Btheta}} \Normal(x\cond y,\sigma_n^2) \Normal(y\cond 0,s^2_{\Btheta}) \ud x \ud y \\
&= \frac{1}{2\pi \sigma_n s_{\Btheta}} \int_{-\infty}^{0} \int_{-\infty}^{\epsilon - m_{\Btheta}} \exp\left( -\half\left[ \frac{(x-y)^2}{\sigma_n^2} + \frac{y^2}{s_{\Btheta}^2} \right] \right) \ud x \ud y \\
&= \frac{1}{2\pi \sigma_n s_{\Btheta}} \int_{-\infty}^{0} \int_{-\infty}^{\epsilon - m_{\Btheta}} \exp\left( -\half \begin{bmatrix}x\\y\end{bmatrix}\T\!\begin{bmatrix}s^2_{\Btheta}+\sigma_n^2&s^2_{\Btheta}\\s^2_{\Btheta}&s^2_{\Btheta}+\sigma_n^2\end{bmatrix}^{-1}\!\begin{bmatrix}x\\y\end{bmatrix}\right) \ud x \ud y \\
&= \Phi_2\left( \begin{bmatrix}\epsilon - m_{\Btheta}\\0\end{bmatrix} \Big| \begin{bmatrix}0\\0\end{bmatrix}, \begin{bmatrix}s^2_{\Btheta}+\sigma_n^2&s^2_{\Btheta}\\s^2_{\Btheta}&s^2_{\Btheta}+\sigma_n^2\end{bmatrix} \right) \\
&= \bvn\left( \frac{\epsilon - m_{\Btheta}}{\sqrt{\sigma_n^2+s^2_{\Btheta}}},0;\frac{s_{\Btheta}}{\sqrt{\sigma_n^2+s^2_{\Btheta}}} \right),
\end{align}
where $\Phi_2$ denotes the bivariate Normal cdf and $\bvn(a,b;\rho)$ denotes the zero-mean bivariate Normal cdf with unit variances and correlation coefficient $\rho$ evaluated at $[a,b]\T$. Finally, using a connection between bivariate Gaussian cdf and \owen{} \cite{Owen1956}, we obtain
\begin{align}
    &\bvn\left( \frac{\epsilon - m_{\Btheta}}{\sqrt{\sigma_n^2+s^2_{\Btheta}}},0;\frac{s_{\Btheta}}{\sqrt{\sigma_n^2+s^2_{\Btheta}}} \right) 
    = T\left( \frac{\epsilon - m_{\Btheta}}{\sqrt{\sigma_n^2+s^2_{\Btheta}}}, \frac{s_{\Btheta}}{\sigma_n} \right) + \half\Phi\left( \frac{\epsilon - m_{\Btheta}}{\sqrt{\sigma_n^2+s^2_{\Btheta}}}\right).
\end{align}
When we combine the equations, we see that the $\Phi(\cdot)$-terms cancel out and we obtain (\ref{eq:eimad}).
\end{proof}

\begin{proof}[Proof of Proposition \ref{prop:eimad}]
The formula for the \eiv{} can be derived in a straightforward manner by combining the GP lookahead formulas given by Lemma 5.1 in \citet{Jarvenpaa2019_sl} with the proof of Proposition 3.2 in \citet{Jarvenpaa2018_acq}. 

The case of \eimad{} requires some extra work. 
First, using an equation from the proof of Lemma \ref{lemma:imad}, we obtain
\begin{align}
    &\mean_{\BDelta\!^*\cond\Btheta^*,D_t} \mathcal{L}^{\textnormal{m}}(\probmeas_{D_t\cup D^*}^f) \\
    &= \mean_{\BDelta\!^*\cond\Btheta^*,D_t} \int_{\Theta} \pi(\Btheta) \Bigg[ 2\!\int_{-\infty}^{0} \Phi\left( \frac{\epsilon - m^*_{t+b}(\Btheta) - y}{\sigma_n}\right) \Normal(y\cond 0,s^{*2}_{t+b}(\Btheta)) \ud y 
    %
    - \Phi\Bigg( \frac{\epsilon - m^*_{t+b}(\Btheta)}{\sqrt{\sigma_n^2+s_{t+b}^{*2}(\Btheta)}}\Bigg) \Bigg] \ud\Btheta \\
    \begin{split}
    &= \int_{\Theta} \pi(\Btheta) \Bigg[ 2\!\int_{-\infty}^{0} \mean_{\BDelta\!^*\cond\Btheta^*,D_t} \Phi\left( \frac{\epsilon - m^*_{t+b}(\Btheta) - y}{\sigma_n}\right) \Normal(y\cond 0,s^{*2}_{t+b}(\Btheta)) \ud y \\
    %
    &\myquad- \mean_{\BDelta\!^*\cond\Btheta^*,D_t} \Phi\Bigg( {\frac{\epsilon - m^*_{t+b}(\Btheta)}{\sqrt{\sigma_n^2+s_{t+b}^{*2}(\Btheta)}}}\Bigg) \Bigg] \ud\Btheta.
    \end{split}
\end{align}
Note that $*$ in $m^*_{t+b}(\Btheta)$ and $s^{*2}_{t+b}(\Btheta)$ is used to emphasise that these quantities depend on $\BDelta\!^*$ and/or $\Btheta^*$.
Since $s_{t+b}^{*2}(\Btheta) = s_t^2(\Btheta) - \Deltav_t(\Btheta;\Btheta^*)$, i.e.~the reduction of the GP variance function $\Deltav_t(\Btheta;\Btheta^*)$ at $\Btheta$ due to the $b$ extra evaluations $D^*$ is deterministic and depends only on $\Btheta^*$ (and not on $\BDelta\!^*$), we obtain for each $\Btheta$ that
\begin{align}
    &\mean_{\BDelta\!^*\cond\Btheta^*,D_t} \Phi\left( \frac{\epsilon - m^*_{t+b}(\Btheta) - y}{\sigma_n}\right) \Normal(y\cond 0,s^{*2}_{t+b}(\Btheta)) \\
    &= \int_{-\infty}^{\infty} \Phi\left( \frac{\epsilon - y - m^*_{t+b}(\Btheta)}{\sigma_n}\right) \Normal(m^*_{t+b}(\Btheta)\cond m_{t}(\Btheta),\Deltav_t(\Btheta;\Btheta^*)) \ud m^*_{t+b}(\Btheta) \Normal(y\cond 0,s_t^2(\Btheta) - \Deltav_t(\Btheta;\Btheta^*)) \\
    &= \Phi\left( \frac{\epsilon - m_{t}(\Btheta) - y}{\sqrt{\sigma_n^2 + \Deltav_t(\Btheta;\Btheta^*)}}\right) \Normal(y\cond 0,s_t^2(\Btheta) - \Deltav_t(\Btheta;\Btheta^*)),
\end{align}
where we have used Lemma 5.1 by \citet{Jarvenpaa2019_sl} and (\ref{eq:helpformula1}). Similarly, we see that
\begin{align}
    \mean_{\BDelta\!^*\cond\Btheta^*,D_t} \Phi\left( \frac{\epsilon - m^*_{t+b}(\Btheta)}{\sqrt{{\sigma_n^2+s_{t+b}^{*2}(\Btheta)}}}\right) 
    %
    %
    = \Phi\left( \frac{\epsilon - m_{t}(\Btheta)}{\sqrt{\sigma_n^2+s_{t}^{2}(\Btheta)}}\right). 
\end{align}
The result now follows by proceeding as in the second part of the proof of Lemma \ref{lemma:imad}. 
\end{proof}

\begin{proof}[Proof of Proposition \ref{prop:lcbconnection}]
\citet{Jarvenpaa2018_acq} showed that the $\alpha$-quantile for $\tildepiabc(\Btheta)$ at any fixed $\Btheta\in\Theta$ is given by
\begin{align}
    z_{t,\alpha}(\Btheta) = \pi(\Btheta)\Phi\left( \frac{s_t(\Btheta)\Phi^{-1}(\alpha) - m_t(\Btheta) + \epsilon}{\sigma_n} \right). \label{eq:quantile}
\end{align}
Using this fact when $\pi(\Btheta)$ is assumed a constant in $\Theta$ shows that
\begin{align}
    \Btheta\opt &= \argmax{\Btheta^*\in\Theta}z_{t,\alpha}(\Btheta^*) \\
    &= \argmax{\Btheta^*\in\Theta}\{ s_t(\Btheta^*)\Phi^{-1}(\alpha) - m_t(\Btheta^*) \} \\
    &= \argmin{\Btheta^*\in\Theta}\{ m_t(\Btheta^*) - \Phi^{-1}(\alpha)s_t(\Btheta^*) \}. \label{eq:last}
\end{align}
Comparison of (\ref{eq:last}) and the LCB acquisition function $\textnormal{LCB}(\Btheta^*) = m_t(\Btheta^*) - \beta_t s_t(\Btheta^*)$ shows immediately that these coincide if $\beta_t=\Phi^{-1}(\alpha)$ i.e.~if $\alpha = \Phi(\beta_t)$. 
\end{proof}

\subsection{On the Gaussian assumptions} \label{appsubsec:gaussianity}

We justify the seemingly strong Gaussianity assumption of the discrepancy $\Delta_{\Btheta}$. 
We briefly analyse a typical case where the discrepancy is formed as a Mahalanobis distance 
\begin{equation}
\Delta_{\Btheta} = \sqrt{(\Bsobs - \Bs_{\Btheta})\T\BW(\Bsobs - \Bs_{\Btheta})}, 
\label{eq:mahald}
\end{equation}
where $\BW\in\reals^{d\times d}$ is a positive definite, $\Bsobs\eqdef s(\Bxobs), \Bs_{\Btheta}\eqdef s(\Bx_{\Btheta})$, and $s:\mathcal{X}\rightarrow\reals^d$ is the summary statistics function usually with $d\geq p$. Recall that $p$ is the dimension of the parameter space $\Theta$. 
If we assume\footnote{This assumption is also made in the synthetic likelihood method \cite{Wood2010,Price2018}.} $\Bs_{\Btheta}$ is jointly Gaussian for each $\Btheta$, some $\Btheta'$ in the posterior modal area satisfies $\Bs_{\Btheta'}\sim\Normal(\Bsobs,\BSigma_{\Btheta'})$ with positive definite $\BSigma_{\Btheta'}$ and if we further choose $\BW = \BSigma^{-1}_{\Btheta'}$, then $\Delta^2_{\Btheta'}\sim\chi^2(d)$, the chi-squared distribution with degree of freedom $d$. 
This follows by noticing that there exists $L_{\Btheta'}\in\reals^{d\times d}$ such that $\BSigma_{\Btheta'}=L_{\Btheta'}L_{\Btheta'}\T$, because $L_{\Btheta'}^{-1}(\Bsobs-\Bs_{\Btheta'})\sim\Normal(\Bzeros,\Id)$ and because the chi-squared distribution $\chi^2(d)$ can be characterised as a sum of squares of $d$ independent standard Normal random variables. 
Further, using the last-mentioned fact, the central limit theorem (CLT), the delta method and the obvious fact that the square root is a smooth function, one can reason that 
$\Delta_{\Btheta'} = (\Delta^2_{\Btheta'})^{1/2}$ 
is approximately Gaussian for large enough $d$. In fact, $\Delta_{\Btheta'}\sim\chi(d)$, the chi distribution with degree of freedom $d$, which is fairly close to Gaussian distribution already with $d=5$.

If $\Bs_{\Btheta'} - \Bsobs$ has nonzero mean and/or $\BW \neq \BSigma^{-1}_{\Btheta'}$, then $\Delta^2_{\Btheta'}$ is no longer chi-squared distributed but follows generalised chi-squared distribution. Detailed analysis of this general case seems difficult. However, if we further assume that the individual summaries, i.e.~the elements of $\Bs_{\Btheta'}$, are independent, and if $\BW$ is diagonal and scales $\Bs_{\Btheta'} - \Bsobs$ so that its elements do not have too variable means and variances which are requirements for a sensible discrepancy function \cite{Prangle2017}, then CLT (with Lindeberg or Lyapunov condition) and delta method might apply so that the approximate Gaussianity still holds for large enough $d$. In this case, the Gaussianity assumption of $\Bs_{\Btheta'}$ is in fact not necessary.

While $\sigma_n^2$ can be heteroscedastic, i.e.~depend on $\Btheta$ as empirically investigated by \citet{Jarvenpaa2018_aoas}, we can expect by continuity that it is often approximately constant on the modal area of the posterior where the GP fit only needs to be accurate. 
%
Also, while the discrepancy is not exactly Gaussian because $\Delta_{\Btheta}$ in (\ref{eq:mahald}) is obviously non-negative, the amount of probability mass of the Gaussian density on the negative values of $\Delta_{\Btheta}$ will typically be very small. 
Finally, while the analysis of this section and our empirical investigations shown in Fig.~\ref{fig:gaus} support the Gaussian assumption, for a particular problem at hand and as in all Bayesian modelling, the goodness of the model fit should be assessed. 

\begin{figure*}[hbt!] 
\centering
\includegraphics[width=0.8\textwidth]{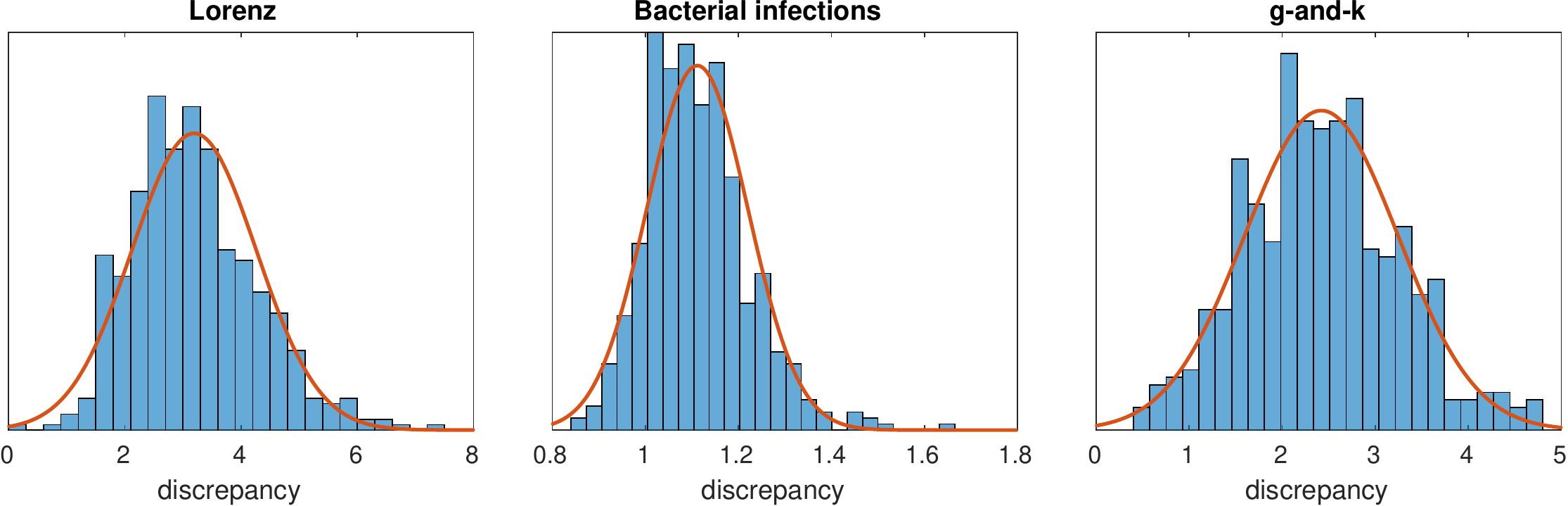}
\caption{Empirical distributions of the discrepancy of the three real-world problems used in this work at their true parameter values. The histogram shows the discrepancy values corresponding to $500$ simulations and the red line shows corresponding Gaussian densities. The discrepancy for the Lorenz and g-and-k model model is formed as a Mahalanobis distance (see Section \ref{app:add_exp} for details). It is seen that the Gaussian assumption is reasonable. } \label{fig:gaus}
\end{figure*}

\section{Additional details on implementation and experiments} \label{app:implementation}

\subsection{Implementation details} \label{app:impl}

We present additional implementation details of our inference algorithm. 
%
The batch-sequential \eiv{} method is shown as Algorithm \ref{alg:main}. 
Other methods for acquiring evaluation locations (\eimad{}, \maxv{}, \maxmad{} and RAND) can be used similarly.
The accuracy of the resulting ABC posterior can be assessed as described in Section \ref{sec:post_uncertainty} either at each iteration (e.g.~immediately after line 15) or only finally (line 19). 

{\centering
\begin{minipage}{.78\linewidth}
\begin{algorithm}[H]
\caption{\babc{} using \eiv{} with synchronous batch design \label{alg:main}} 
 \begin{algorithmic}[1]
 \Require Prior $\pi(\Btheta)$, simulation model $\pi(\Bx \cond \Btheta)$, GP prior $\probmeas^{f}$, discrepancy $\Delta(\Bxobs,\Bx)$, 
 batch size $b$, initial batch size $b_0$, max.~iterations $\iter_{\text{max}}$, 
 number of IS samples $s_{\text{IS}}$, number of MCMC samples $s_{\text{MC}}$ 
 %
 \For{$r=1:b_0$} \Comment{\textbf{Can be run in parallel.}}
 \State Sample $\Btheta_{r} \simiid \pi(\cdot)$ \Comment{{Space filling designs~can be alternatively used.}}
 \State Simulate $\Bx_{r} \simiid \pi(\cdot \cond \Btheta_{r})$ and compute $\Delta_r = \Delta(\Bxobs,\Bx_{r})$
 \EndFor
 \State Set $D_{b_0} \leftarrow \{(\Delta_{r}, \Btheta_{r})\}_{r=1}^{b_0}$
 %
 \For{$\iter=1:\iter_{\text{max}}$} 
  %
  \State Obtain GP hyperparameters $\Bpsi_{\text{MAP}}$ using $D_{b_0+(\iter-1)b}$ 
  %
  \State Sample $\Btheta^{(j)}\sim\pi_q$ using MCMC and compute IS weights $\omega^{(j)}$ for $j = 1,\ldots,s_{\text{IS}}$ 
  %
  \For{$r=1:b$} \Comment{{Batch is constructed using greedy optimisation.}}
  	\State Obtain $\Btheta_{r}^{*}$ as the minimiser of the IS approximation of $L_t^{\textnormal{v}}([\Btheta_{1:r-1}^{*},\Btheta_{r}^{*}])$ in (\ref{eq:eiv})
  \EndFor
  \For{$r=1:b$} \Comment{\textbf{Can be run in parallel.}}
 	\State Simulate $\Bx_{r}^{*} \simiid \pi(\cdot \cond \Btheta_{r}^{*})$ and compute $\Delta_r^* = \Delta(\Bxobs,\Bx_{r}^{*})$
 \EndFor
  \State Update training data $D_{b_0+ib} \leftarrow D_{b_0+(i-1)b} \cup \{(\Delta_{r}^{*}, \Btheta_{r}^{*})\}_{r=1}^b$
 \EndFor
 \State Obtain GP hyperparameters $\Bpsi_{\text{MAP}}$ using $D_{b_0 + \iter_{\text{max}}b}$
 \State Sample $\Bvartheta^{(1:s_{\text{MC}})}$ from (\ref{eq:pimean}) using MCMC \Comment{(\ref{eq:pimed}) can be alternatively used.}
 \State \Return Samples $\Bvartheta^{(1:s_{\text{MC}})}$ from the approximate ABC posterior
 \end{algorithmic}
\end{algorithm}
\end{minipage}
\par
}
\vspace{.4cm}

When the dimension of the parameter space $p>2$, we used the adaptive MCMC method by \citet{Haario2006} to sample from the model-based estimates of the ABC posterior (line 18) and from the instrumental densities needed for the IS approximation of \eiv{} and \eimad{} acquisition functions (line 8). Adaptive MCMC was also used for the IS approximation needed for ABC posterior uncertainty quantification. In all of these cases, we run multiple chains initialised at the point with the highest log-density value computed over the current points in $D_t$. The first half of each chain was neglected as burn-in and the chains were then combined and thinned. In 2D, similar grid-based numerical computations were used instead. 

When sampling from the model-based estimate of the ABC posterior (line 18), the samples were thinned to the size of $10^4$ and kernel density estimation was used to estimate the (marginal) densities from the resulting samples. For the grid-based numerical computations in 2D, we used $100\times 100$ grid of points.

To evaluate \eiv{} and \eimad{} acquisition functions, we first sampled from the instrumental density (denoted as $\pi_q$ on line 8 of Algorithm \ref{alg:main}) which is the current loss surface interpreted as a pdf as mentioned in Section \ref{subsec:opt}. These samples were thinned to the size of $500$ points used for computing the normalised importance weights $\omega^{(j)}$. In 2D, $50\times 50$ grid-based computations were used instead. 
The same instrumental density and thus the same set of importance samples was used for greedily optimising each point in the batch (line 10) although it is also possible to use different instrumental densities. 
%
The global optimisation of the acquisition functions was performed by first using random search (with $1000$ points in 2D and $2000$ in 3D and 4D) to roughly locate good regions and then improving the best $10$ points found this way by initialising gradient-based algorithm at these points.\footnote{We used \texttt{fmincon} in MATLAB. The gradient was approximated by finite differences for simplicity but analytical gradient computations could be also used to improve optimisation.} The best point evaluated was taken as the optimal solution. While other optimisation strategies are also possible, our method already produced good results.

We used the following settings for the uncertainty quantification of the ABC posterior in Section \ref{sec:post_uncertainty}: The 2D integrals over $\Theta$ were computed numerically in a $80\times 80$ grid, i.e.~we used $\bar{n}=80$ producing $6400$ grid points. For $p>2$, we used the adaptive MCMC with $15$ chains each with length $20000$. The chains were finally combined and thinned to $\tilde{n}=7500$ representative points for computing the importance weights. We used $s=2000$ GP sample paths. Marginal densities for e.g.~Fig.~\ref{fig:bacterial_post_uncertainty} were computed from the resulting weighted sample sets using weighted kernel density estimation.

\subsection{Computation times for optimising the acquisition functions}

The computational cost of evaluating the acquisition functions of Section \ref{sec:parallel} depends on various factors. 
We here report computation times\footnote{These times were obtained on a standard laptop with Intel Core i5 2.3GHz CPU and 8Gb RAM.} of our MATLAB implementation\footnote{\owen{} values were computed using an efficient C-implementation of the algorithm by \citet{Patefield2000}.} when the simulation budget is $t=810$ in 2D (Multimodal toy model) and $t=820$ in 4D (g-and-k model). We report the computation times at both the first and the last iteration. These show the minimum and maximum costs, respectively. 
%

In 2D, where grid-based numerical computations were used, sequential \maxv{} required $0.3-1.5$s and its batch version $5-35$s for constructing the whole batch of size $b=5$. In 4D, the computation times roughly doubled. 
In 2D, sequential \eiv{} required $2.5-13$s and its batch version $18-80$s for the whole batch of size $b=5$. In 4D, these times were $9-80$s and $27-250$s, respectively. The computation time of \eiv{} scales better than linearly for $b$ in 4D because we sample once from the instrumental density in the beginning and re-use the same importance weights for selecting each point in the current batch. In 2D, this scaling is roughly linear. 

The difference in computation times between \maxmad{} and \maxv{}, as well as between \eimad{} and \eiv{}, was small. This is because the computation costs are dominated by the GP-based computations and evaluations of the \owen{} needed for both. 
Finally, we emphasise that while the GP computations and the optimisation of the acquisition function are not particularly cheap, the simulation times for realistic models typically dominate the total cost. The reported computation times can be also reduced by more efficient implementation. 
However, if running the simulation model is very fast (e.g.~less than a fraction of a second), standard ABC methods should be preferred even if they require substantially more simulations.

\subsection{Additional details on experiments} \label{app:add_exp}

We describe additional details of the experimental set-up. 
Fig.~\ref{fig:syn} visualises the four synthetically constructed 2D posteriors used in Section \ref{subsec:synth}. These examples were taken from \citet{Jarvenpaa2018_acq} where further details can be found. 

\begin{figure*}[hbt!] 
\centering
\includegraphics[width=0.93\textwidth]{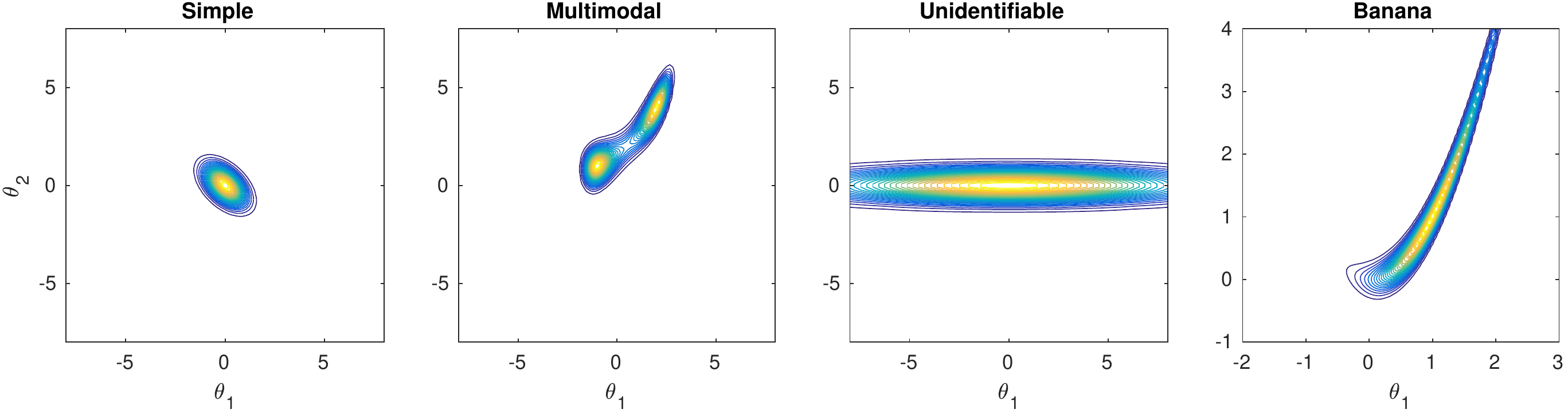}
\caption{Synthetic 2D posterior densities used in the experiments of Section \ref{subsec:synth}.} \label{fig:syn}
\end{figure*}

We used ABC-MCMC to obtain the ground truth ABC posterior. The algorithm was initialised with the true value or, in the case of the bacterial infections model, using a point estimate from earlier studies \cite{Numminen2013}. The proposal density for ABC-MCMC was hand-tuned. For Lorenz model we used $8$ chains with length $3\cdot 10^6$ and for g-and-k model $8$ chains with length $10^7$ samples. For bacterial infections model we used $20$ chains with length $7.5\cdot 10^4$ samples. The chains were finally combined and thinned to $10^4$ samples to represent the ground truth ABC posterior. 

Mahalanobis distance as in Eq.~\ref{eq:mahald} was used as the discrepancy for Lorenz and g-and-k models. The simulation model was run $500$ times to estimate the covariance matrix of the summary statistics at the true parameter and the matrix $\BW$ was chosen to be the inverse of the covariance matrix. Of course, such discrepancy is unavailable in practice because the true parameter is unknown and the computational budget limited. However, as the main goal of this paper is to approximate any given ABC posterior with a limited simulation budget, we chose our target ABC posterior this way. For this reason we also fixed $\epsilon$ to small predefined value for each test problem. 
Investigating whether one could adaptively adjust the discrepancy in our \babc{} framework (without using a large number of replicates at each proposed point as is required e.g.~in the synthetic likelihood method \cite{Wood2010}) is left as a topic for future work. 

\citet{Gutmann2016} defined a discrepancy for the bacterial infections model by summing four $L^1$-distances computed between certain individual summaries. For details, see example $7$ in \citet{Gutmann2016}. We used the same discrepancy except that we further took square root of their discrepancy function. We obtained a similar ABC posterior as the original article \cite{Numminen2013} where ABC-PMC algorithm and a slightly different approach for comparing the data sets were used.

\section{Additional results and illustrations} \label{app:experiments}

We show additional results and illustrations of the experiments in Section \ref{sec:experiments}. 
%
Fig.~\ref{fig:acq1} and \ref{fig:acq2} show the evaluation locations and the resulting estimates of the ABC posteriors after $110$ simulations for two synthetic 2D models of Section \ref{subsec:synth}. 

\begin{figure*}[hbt!] 
\centering
\includegraphics[width=0.97\textwidth]{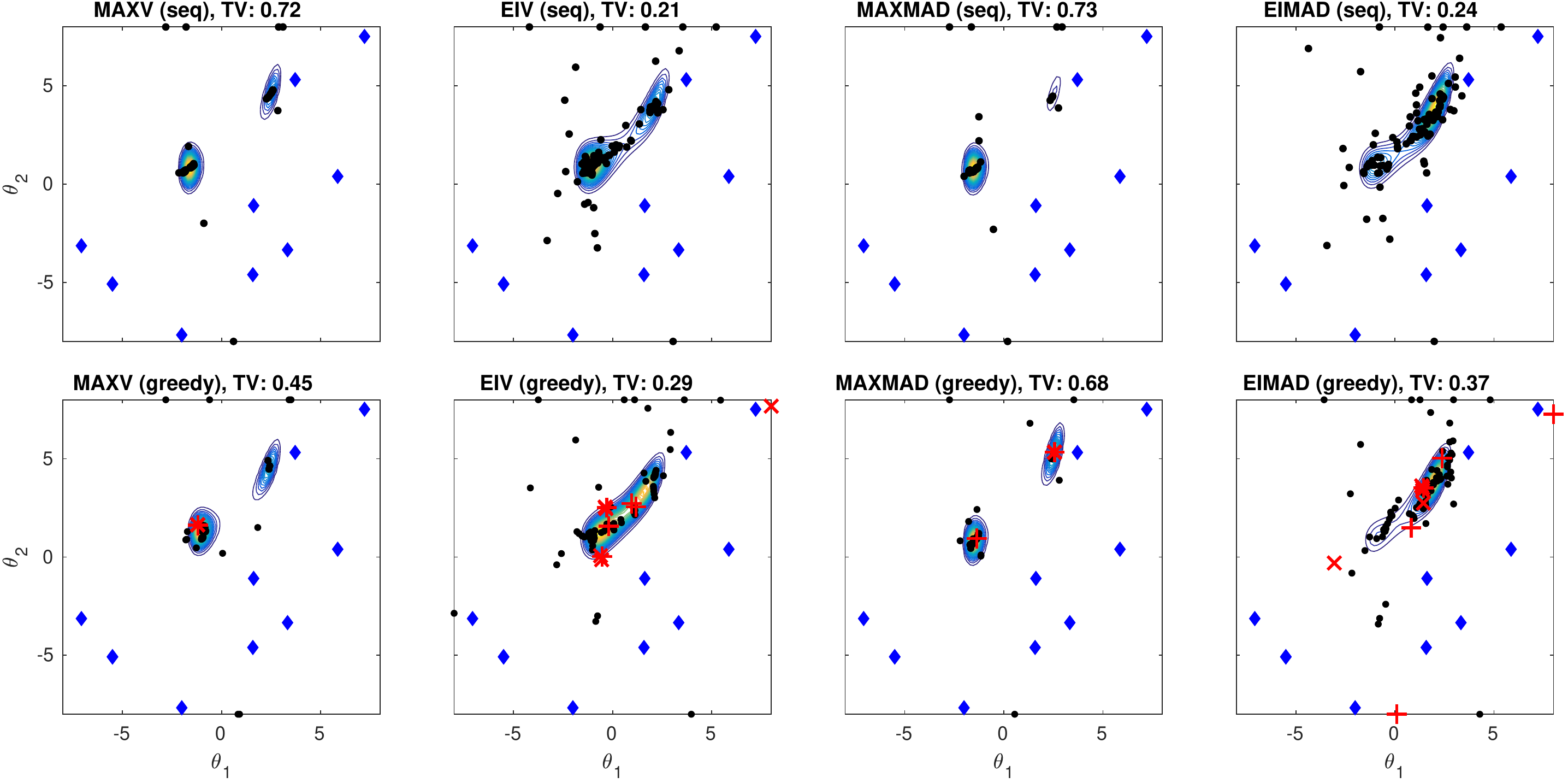}
\caption{Multimodal test problem. The first row shows the sequential methods and the second row the corresponding greedy batch methods. The blue diamonds show the $10$ initial points and the black dots $100$ additional points selected using each acquisition function (the last two batches in the second row are however highlighted by red plus-signs and crosses). TV shows the total variation distance between the true and estimated ABC posteriors for each particular case. } \label{fig:acq1}
\end{figure*}

\begin{figure*}[hbt!] 
\centering
\includegraphics[width=0.97\textwidth]{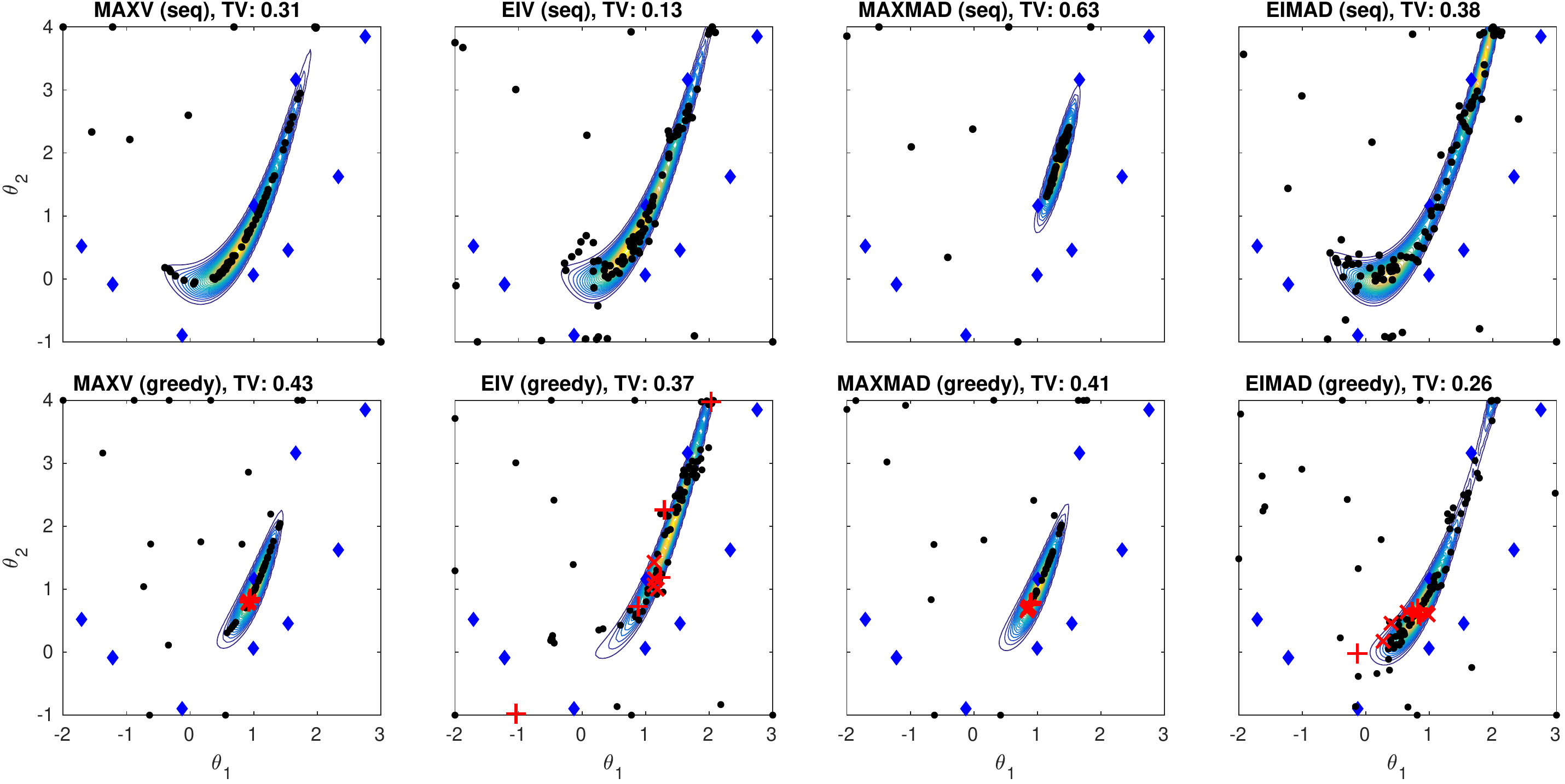}
\caption{Banana test problem. See the caption of Fig.~\ref{fig:acq1} for description. } \label{fig:acq2}
\end{figure*}


Fig.~\ref{fig:lorenz_example_post} and \ref{fig:bacterial_example_post} show typical estimated ABC posterior densities of the Lorenz and bacterial infections models of Section \ref{subsec:simmodels}, respectively. These results are shown to demonstrate the accuracy obtainable with very limited simulations. These particular results were obtained with the sequential \eiv{} method using $600$ iterations corresponding to $610$ simulations (Lorenz model) or $620$ simulations (bacterial infections model). 

\begin{figure}[hbt!] 
\centering
\includegraphics[width=1\textwidth]{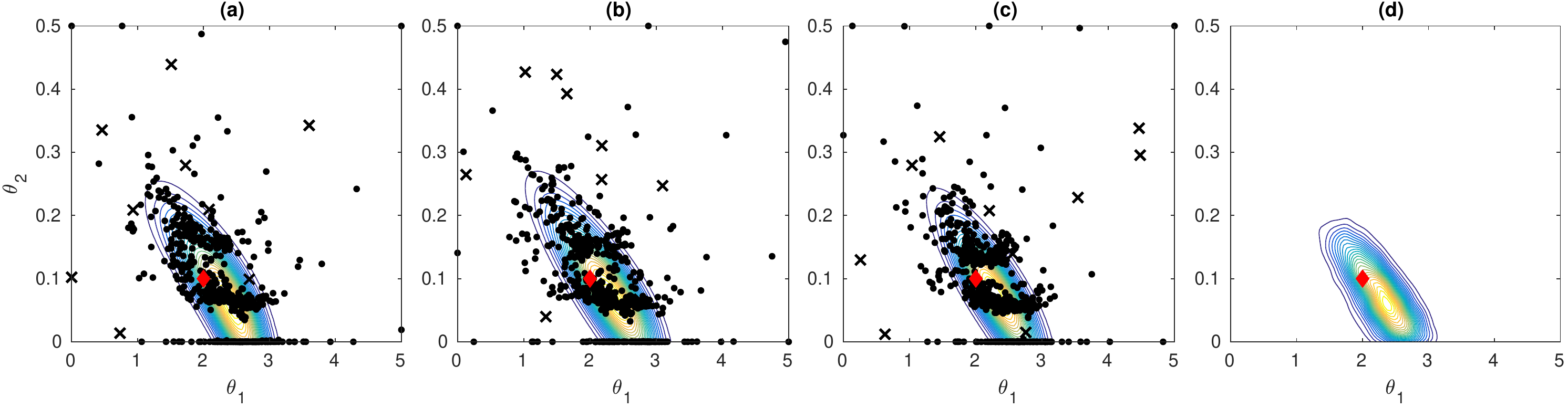}
\caption{Estimated ABC posteriors for the Lorenz model. (a-c) Three typical estimates of the ABC posterior with corresponding simulation locations. Initial locations are shown as black crosses and the ones selected using \eiv{} acquisition function are shown as black dots. The true parameter value used to generate the data is marked with the red diamond. (d) The ground truth ABC posterior computed using ABC-MCMC with extensive simulations.} \label{fig:lorenz_example_post}
\end{figure}

\begin{figure}[hbt!] 
\centering
\includegraphics[width=0.84\textwidth]{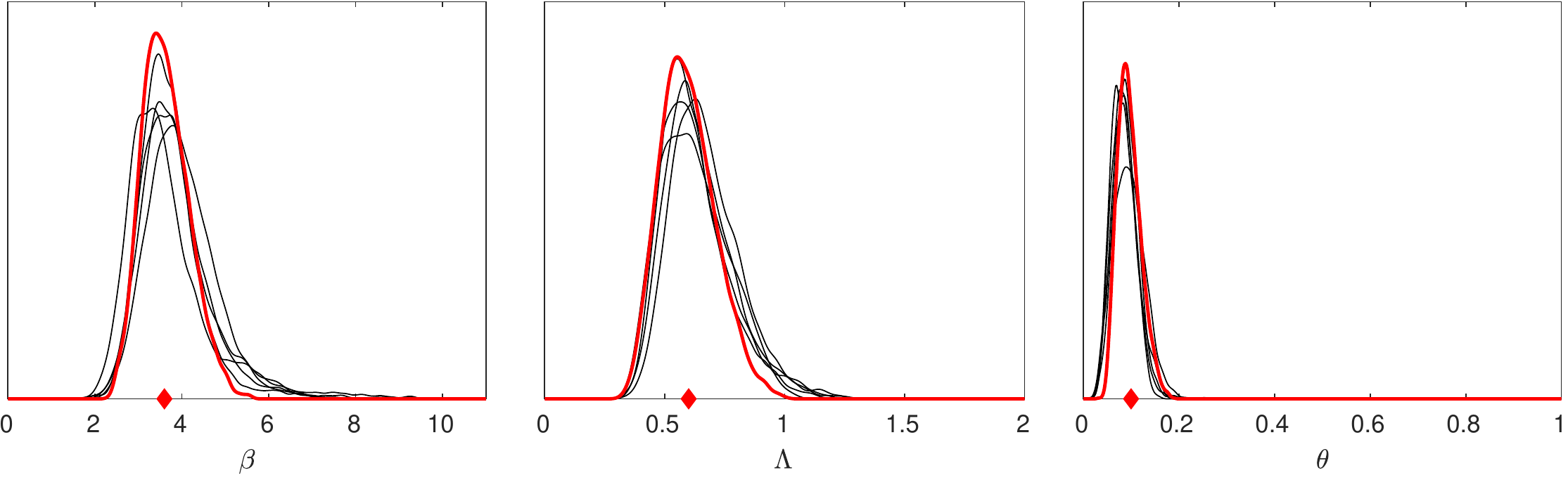}
\caption{Estimated marginal ABC posteriors for the bacterial infections model. Red lines show the ground truth ABC posterior computed using ABC-MCMC with extensive simulations. Black lines show five typical estimated ABC posteriors resulting from different simulation model realisations and the initial sets of simulation locations. The true parameter value used to generate the data is marked with the red diamond.} \label{fig:bacterial_example_post}
\end{figure}

Fig.~\ref{fig:bacterial_post_uncertainty} illustrates the ABC posterior uncertainty quantification for the bacterial infections model. Fig.~\ref{fig:bacterial_post_uncertainty_evol} shows the evolution of the uncertainty of the ABC posterior expectations over $600$ iterations. Sequential \eiv{} method was used and one typical case is shown. 
The results suggest that while the ABC posterior is well estimated at the last iteration, there is some uncertainty left about its exact shape.
Similar observations were also done with g-and-k model of the next section (results not shown). The true value is not always contained in the $95\%$ CI which is likely because the uncertainty in the GP hyperparameters is ignored for simplicity and because the GP is reasonable but imperfect model for the discrepancy. 

Although we used quadratic GP mean function to encode the prior assumption of unimodal posterior, we observed that the uncertainty of the ABC posterior near the boundaries of the parameter space during the early iterations can be high leading to multimodality. Such cases can be difficult for the MCMC as it can fail to locate all the modes or sample sufficiently from them. For this reason, the uncertainty quantification based on the proposed IS approach needs to be interpreted cautiously. 
More sophisticated sampling techniques as considered now here might be useful.

\begin{figure}[hbt!] 
\centering
\includegraphics[width=0.8\textwidth]{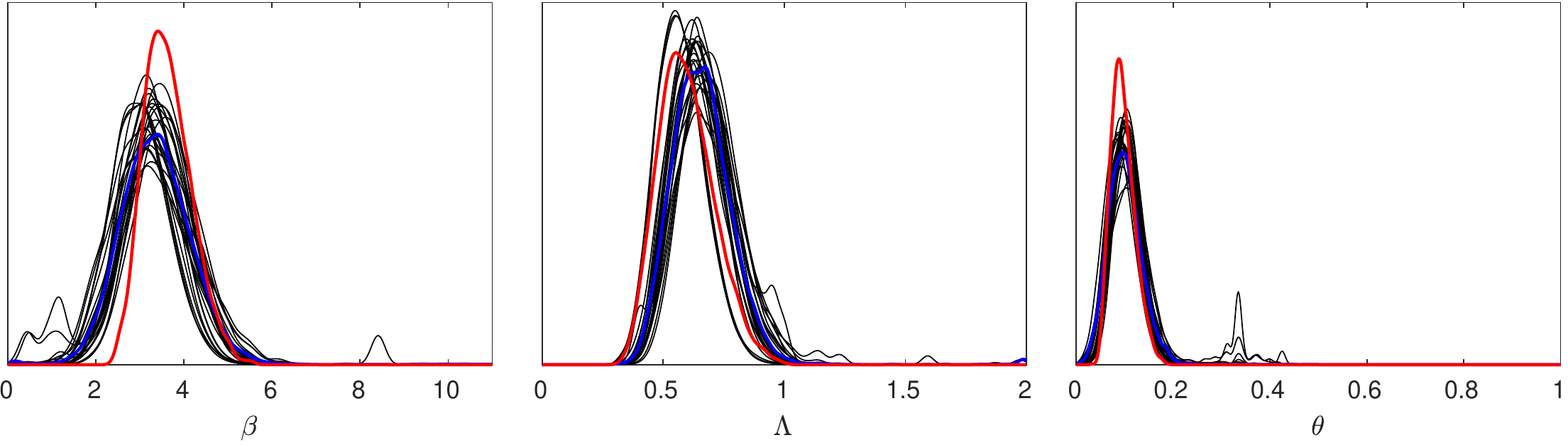} \\ 
\vspace{0.25cm}
\includegraphics[width=0.8\textwidth]{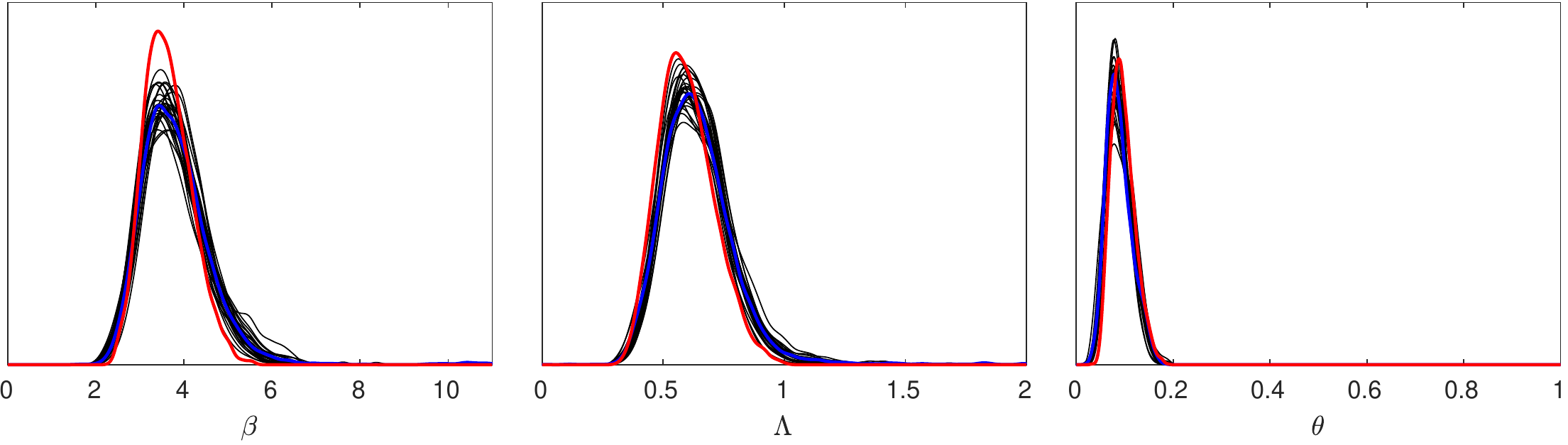}
\caption{Uncertainty quantification for the ABC posterior marginals of the bacterial infections model at the $100$th iteration corresponding $t=120$ simulations (top row) and at the last iteration corresponding $t=620$ simulations (bottom row). Red line shows the ground truth ABC posterior, blue line shows the estimate based on (\ref{eq:pimean}) and the black lines show some sampled ABC marginal posteriors that (approximately) represent the uncertainty due to the limited number of simulations $t$. } \label{fig:bacterial_post_uncertainty}
\end{figure}

\begin{figure}[hbt!] 
\centering
\includegraphics[width=0.85\textwidth]{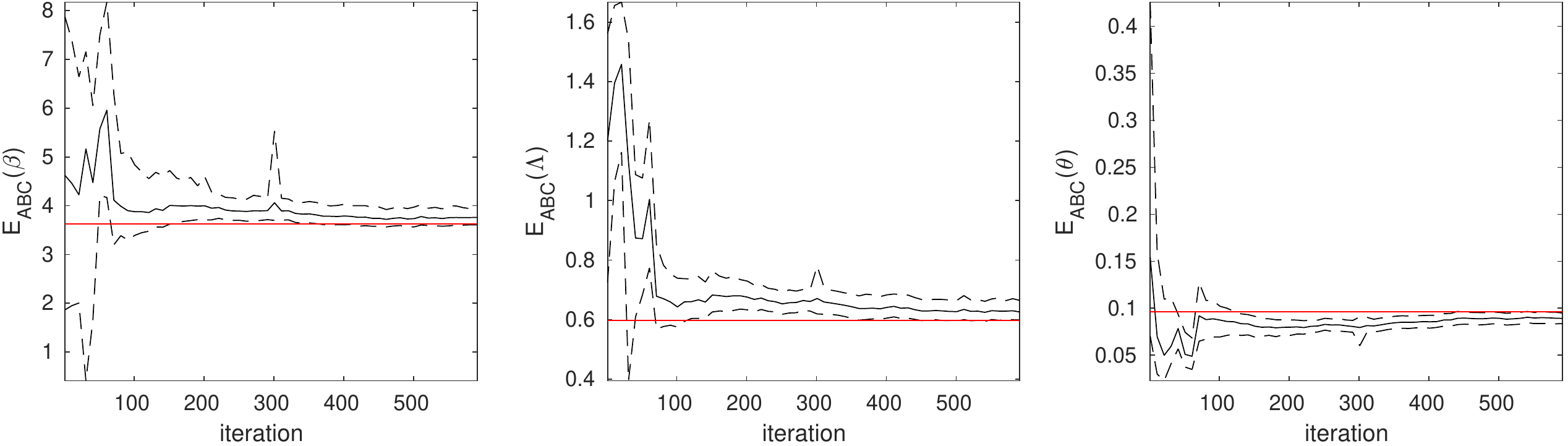}
\caption{Evolution of the uncertainty of the ABC posterior expectations of the bacterial infections model over $600$ iterations corresponding $t=620$ simulations for one typical run of the inference algorithm. This is as Fig.~\ref{fig:lorenz1}(b-c) except that iteration is here not on the log-scale.} \label{fig:bacterial_post_uncertainty_evol}
\end{figure}

\section{Additional experiments: g-and-k model} \label{app:gk}

We present the g-and-k distribution and our additional experiments with this benchmark model. 
The g-and-k model is a probability distribution defined via its quantile function
\begin{equation}
Q(\Phi^{-1}(q);\Btheta) = a + b\left( 1 + c\frac{1-\exp(-g\Phi^{-1}(q))}{1+\exp(-g\Phi^{-1}(q))} \right)(1+(\Phi^{-1}(q))^2)^k \Phi^{-1}(q), \label{eq:gk_eq}
\end{equation}
where $a,b,c,g$ and $k$ are unknown parameters, $q\in[0,1]$ is a quantile and $\Phi^{-1}$ denotes the quantile function of the standard normal distribution. There is no analytical formula for the likelihood but sampling from it is straightforward \cite{Price2018}. 
%
We fix $c=0.8$ as is common in literature and estimate the parameters $\Btheta = (a,b,g,k)$ from $10^4$ samples generated using $\Btheta = (3,1,2,0.5)$ as the true parameter value. We use independent uniform priors $a\sim \Unif([2, 4]), b\sim \Unif([0, 3]), g\sim \Unif([1, 4]), k\sim \Unif([0, 2])$. We consider the four summary statistics defined via an auxiliary model as suggested by \citet{Price2018} and use them to form a Mahalanobis discrepancy function as already described in Section \ref{app:add_exp}. Although the discrepancy is formed only from four summary statistics, we observed that it is very close to Gaussian near the true parameter value, see Fig.~\ref{fig:gaus}. 

The results for the g-and-k model are shown in Fig.~\ref{fig:gk_res1} and Fig.~\ref{fig:gk_res2}. 
The conclusions from the results resemble those of the Lorenz and bacterial infections models in that the proposed batch techniques produce substantial improvements over the corresponding sequential ones. However, the overall approximation errors are slightly larger than for the bacterial model presumably due to more noticeable model misspecification (the variance of the discrepancy is only approximately constant near the true value) and the higher dimensionality of the parameter space. 
Interestingly, in this problem the heuristic \maxv{} method eventually produces the most accurate approximations. However, \eiv{}, producing more conservative estimates, works more reliably if large batch sizes are used.

\begin{figure}[hbtp!] 
\centering
\includegraphics[width=0.44\textwidth]{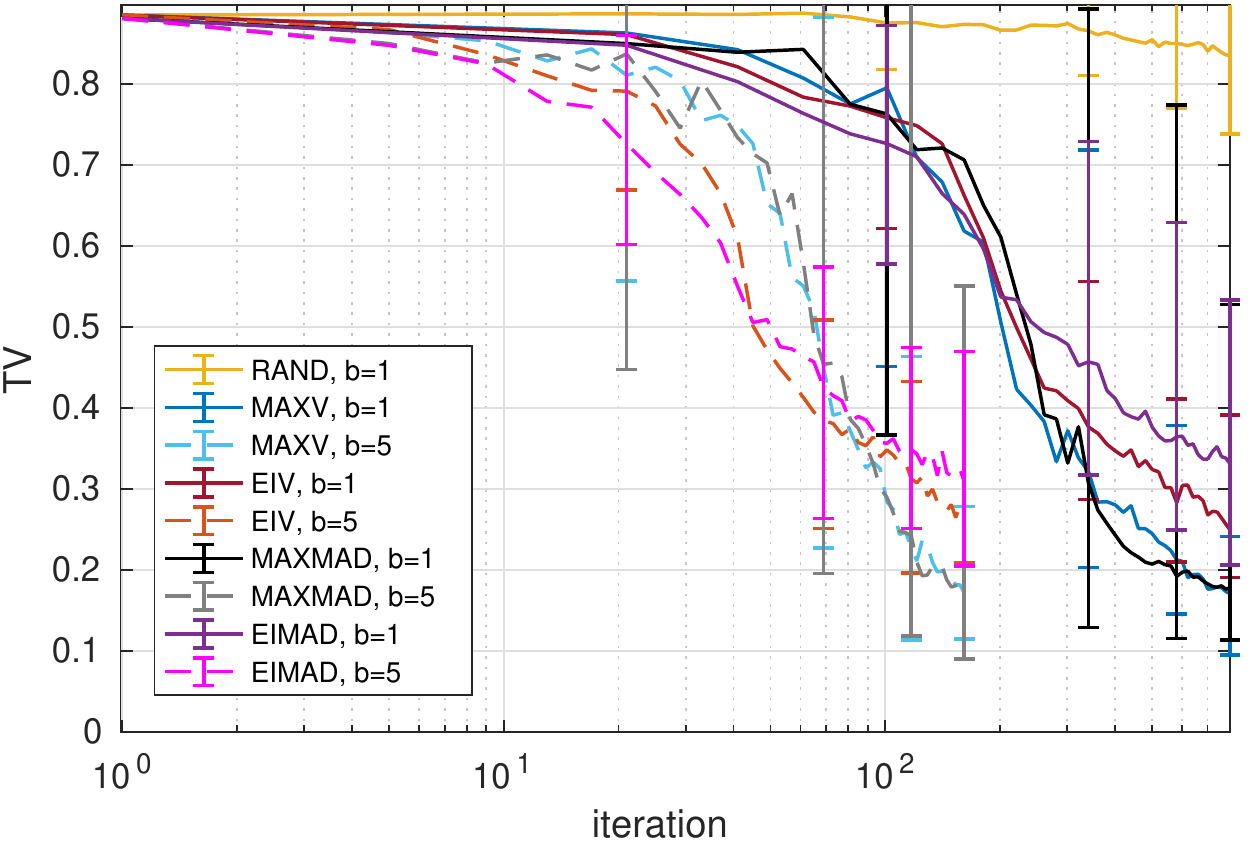}
\caption{Results for the g-and-k model. All proposed methods were tested with two batch sizes.} \label{fig:gk_res1}
\end{figure}

\clearpage 
\begin{figure}[hbtp!] 
\centering
\includegraphics[width=0.42\textwidth]{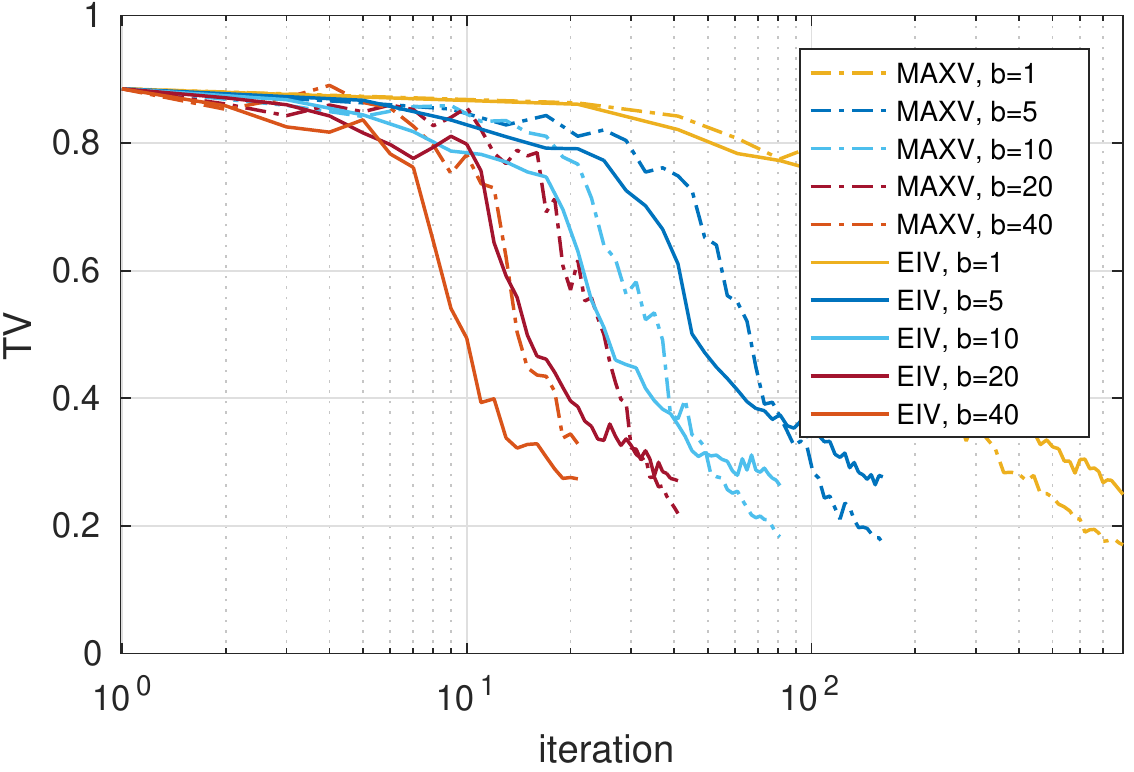}
\caption{Results for the g-and-k model. Further analysis for two methods, \maxv{} and \eiv{}, with varying batch sizes.} \label{fig:gk_res2}
\end{figure}

\end{document}